\documentclass{article}
\usepackage{ijcai22}

\usepackage{times}
\usepackage{soul}
\usepackage{url}
\usepackage[hidelinks]{hyperref}
\usepackage[utf8]{inputenc}
\usepackage[small]{caption}
\usepackage{graphicx}
\usepackage{amsmath}
\usepackage{amsthm}
\usepackage{booktabs}
\usepackage{algorithm}
\usepackage{algorithmic}
\usepackage{microtype}
\usepackage{subfigure}
\usepackage{textcomp}
\usepackage{amsfonts}
\usepackage{multirow}
\usepackage{paralist} 
\usepackage{color} 

\usepackage{array}
\usepackage{threeparttable}
\usepackage{bbm}
\usepackage{amssymb}
\usepackage{mathrsfs} 

\usepackage{paralist}

\newtheorem{theorem}{Theorem}
\newtheorem{assumption}{Assumption}

\newtheorem{corollary}{Corollary}
\newtheorem{remark}{Remark}
\newtheorem{lemma}{Lemma}

\title{Domain Generalization through the Lens of Angular Invariance }

\author{
Yujie Jin$^{1 \ *}$
\and
Xu Chu$^{2}$ \thanks{The first two authors contributed to this work equally.} \and
Yasha Wang$^{1\ \dag}$ \And
Wenwu Zhu$^{2}$ \thanks{Corresponding authors.} 
\affiliations
$^1$ Peking University, Beijing, China\\
$^2$ Tsinghua University, Beijing, China\\
\emails
\{jyj17pku,wangyasha\}@pku.edu.cn,
\{chu\_xu,wwzhu\}@tsinghua.edu.cn
}

\begin{document}

\maketitle

\begin{abstract}
  Domain generalization (DG) aims at generalizing a classifier trained on multiple source domains to an unseen target domain with domain shift. A common pervasive theme in existing DG literature is domain-invariant representation learning with various invariance assumptions. However, prior works restrict themselves to an impractical assumption for real-world challenges: If a mapping induced by a deep neural network (DNN) could align the source domains well, then such a mapping aligns a target domain as well. In this paper, we simply take DNNs as feature extractors to relax the requirement of distribution alignment. Specifically, we put forward a novel angular invariance and the accompanied norm shift assumption. Based on the proposed term of invariance, we propose a novel deep DG method dubbed Angular Invariance Domain Generalization Network (AIDGN). The optimization objective of AIDGN is developed with a von-Mises Fisher (vMF) mixture model. Extensive experiments on multiple DG benchmark datasets validate the effectiveness of the proposed AIDGN method.
\end{abstract}

\section{Introduction}
Over the past few years, supervised deep learning has achieved remarkable success on many challenging visual tasks~\cite{krizhevsky2012imagenet,long2015fully,he2016resnet}. An underlying assumption of the popular supervised DL methods is the identically distributed condition, namely, the generating functions of training data and testing data are identical. We say a \textit{domain shift} exists between the training data (\textit{source domain}) and the testing data (\textit{target domain}) if the identical condition is violated. When there is a domain shift, the favored empirical risk minimization (ERM) learning \cite{vapnik1999erm} would be ill-posed, since the empirical risk over the training data is not guaranteed to converge to the risk of the testing data asymptotically.

\begin{figure}[htbp]
    \centering
    \subfigure[visualization of domains]{
    \begin{minipage}{4.05cm}
    \centering
        \label{erm-domain}
        \includegraphics[width=4.05cm]{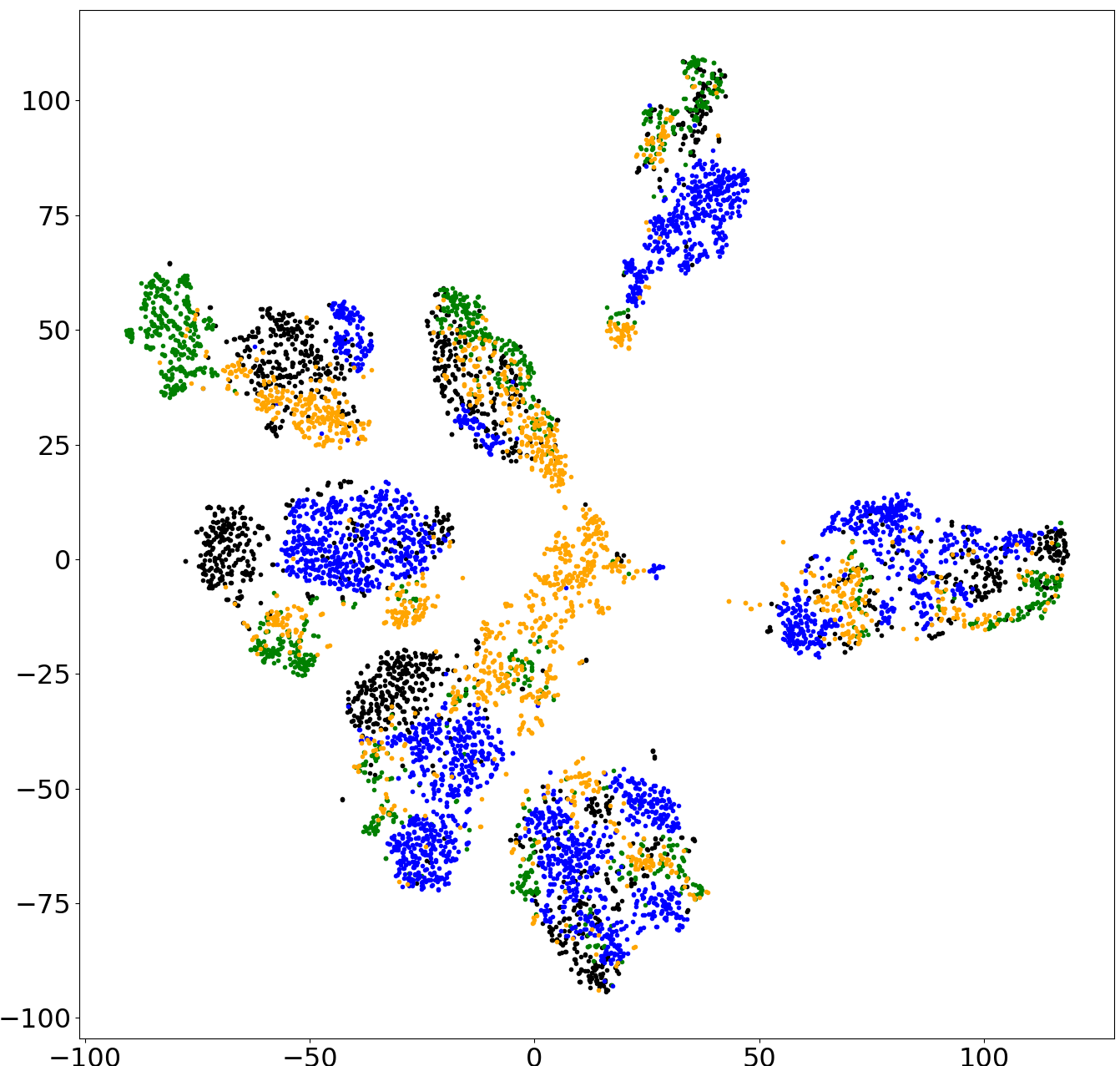}
    \end{minipage}
    }
    \subfigure[visualization of classes]{
    \begin{minipage}{4.05cm}
    \centering
        \label{erm-class}
        \includegraphics[width=4.05cm]{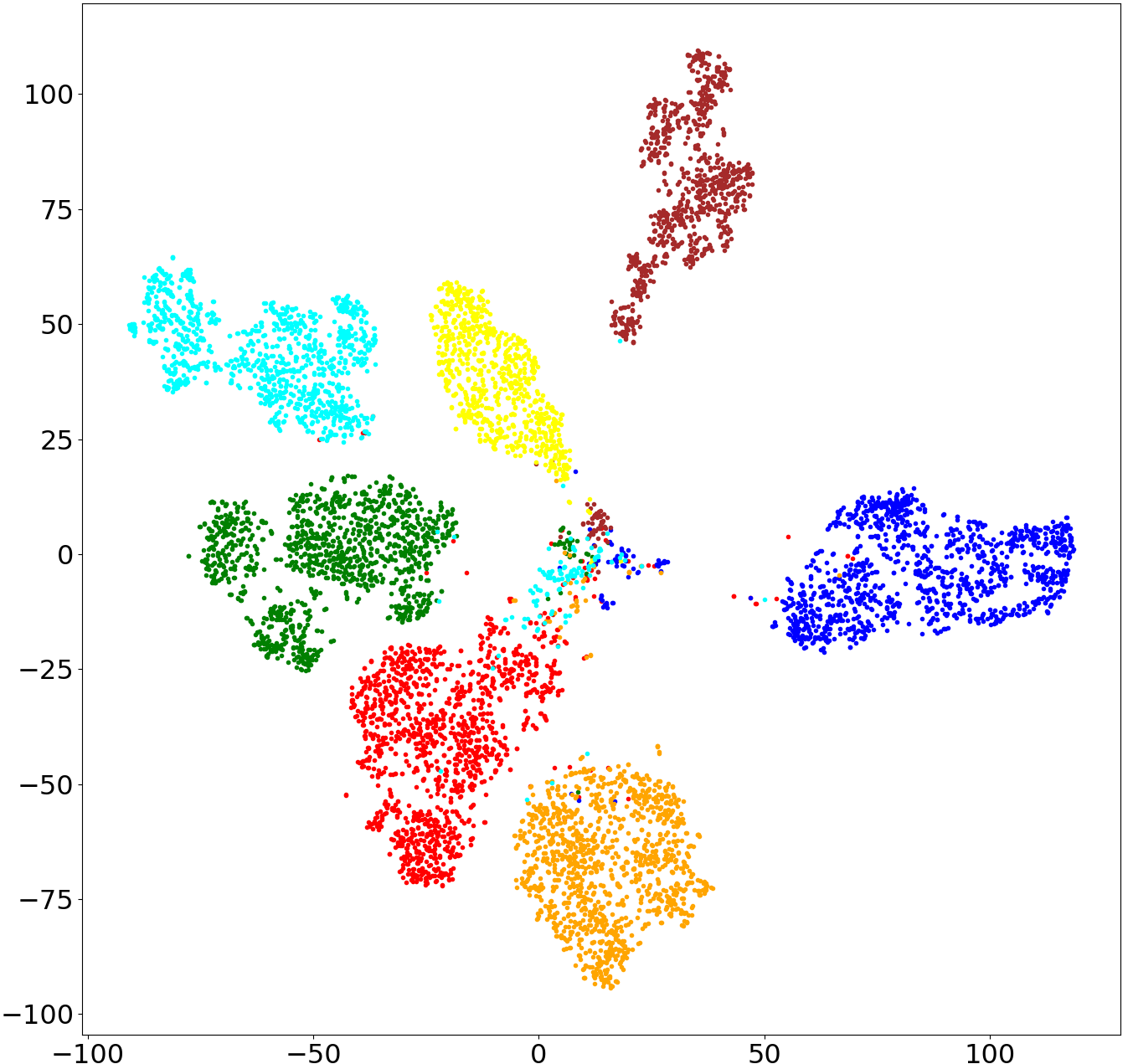}
    \end{minipage}
    }
\caption{Feature visualization for a model trained with ERM on the PACS dataset: (a) different colors indicate  different domains, source domains include cartoon (black), photo (green) and sketch (blue) while the target domain is art-painting (orange); (b) different colors represent different classes. 
\textit{Best viewed in color (Zoom in for details).}}
\label{erm-visual}
\end{figure}

Domain generalization (DG) aims at generalizing the model trained on multiple source domains to perform well on an unseen target domain with domain shift~\cite{blanchard2011generalizing}
. The inductive setting of DG assumes no target data is available during training, differentiating DG from the transductive domain adaptation methodss~\cite{ben2007analysis}
, thus making DG more 
practical and challenging. 

Intuitively, in order to carry out a successful knowledge transfer from ``seen" source domains to an ``unseen" target domain, there have to be some underlying similarities among these domains. From a theoretical standpoint, invariance among the distributions of domains should be investigated. To this end, a predominant stream in DG is domain-invariant representation learning, with various invariance and shift assumptions such as \textit{covariate shift} assumption~\cite{li2018mmd}, \textit{conditional shift} assumption~\cite{li2018cdann}, and \textit{label shift} assumption~\cite{liu2021ijcai21}. 
However, prior works overemphasize the importance of joint distribution alignment under an impractical assumption: an injective mapping (implies a tendency of losing class discriminative information) aligning the source joint distributions on the induced space could align the target joint distribution as well. An easy counter-example is a constant mapping that aligns any distributions on the induced space. 
Recently, theoretical analysis has revealed a fundamental trade-off between achieving well-alignment and low joint error \cite{zhao2019learning} on various domains. Empirically, a study~\cite{gulrajani2020domainbed} observed limited performance gain of those invariant learning methods over ERM under a fair evaluation protocol, demonstrating the difficulty of balancing alignment and generalization.

 In this paper, we take a step back from pursuing 
 domain alignment. We model the relative difference between the target domain and each source domain instead. Specifically, we put forward a novel paradigm of domain shift assumption: the \textbf{angular invariance} and \textbf{norm shift} assumption. The proposed assumption says that under the polar reparameterization~\cite{blumenson1960derivation}, the relative difference between the DNN push forward measures is captured by the norm parameters and invariant to the angular parameters. The insight of angular invariance and norm shift is inspired by the acknowledged fact that the internal layers in DNNs capture high-level semantic concepts (e.g., eye, tail)~\cite{zeiler2014visualizing}
 , which are connected to category-related discriminative features. The angular parameters capture the correlations between the high-level semantic concepts, while the norm parameter captures the magnitude of the high-level semantic concepts.
 In the practice of DG, the DNN feature mapping pre-trained on ImageNet is fine-tuned on the source domains. Therefore the semantic concepts memorized by the internal layers are biased to the source domains, and leading to higher-level of neuron activations. Hence we expect a difference of norm distribution of latent representations between a source domain and a target domain. Meanwhile the correlations between high-level concepts in a fixed category are relatively stable. Thus we expect invariant angular distributions across different domains. 
 We do t-SNE feature visualization on the PACS dataset for an ERM-trained model to motivate and substantiate our assumption. Fig~\ref{erm-domain} shows that the norm distribution of the target domain (orange) significantly differs from that of source domains, while the distributions over angular coordinates are homogeneous. Fig~\ref{erm-class} shows that the learned class clusters are separated well by the angular parameters. 

Apart from the novel angular invariance and norm shift assumption, our methodological contribution is manifested by a novel deep DG algorithm called \textbf{A}ngular \textbf{I}nvariance \textbf{D}omain \textbf{G}eneralization \textbf{N}etwork (\textbf{AIDGN}). The designing principle of the AIDGN method is a minimal level of modification of ERM learning under modest intensity distributional assumptions, such as assuming the distribution families of maximum entropy. Concretely, (1) We show that the angular invariance enables us to compute the marginals over the norm coordinate to compare probability density functions of the target distribution and each source distribution in the latent space. Moreover, we compute the relative density ratio analytically based on the maximum entropy principle~\cite{jaynes1957information}. (2) Within a von-Mises Fisher (vMF) mixture model \cite{gopal2014mises}, we connect the target posterior with the density of each mixture component, re-weighted by the relative density ratio mentioned above and the label densities. (3) We derive a practical AIDGN loss from the target posterior. The deduction adopts the maximum entropy principle for label densities and solves a constrained optimization problem.

We conduct extensive experiments on multiple DG benchmarks to validate the effectiveness of the proposed method and demonstrate that it achieves superior performance over the existing baselines. Moreover, we show that AIDGN effectively balances the intra-class compactness and the inter-class separation,  and thus reduces the uncertainty of predictions.


\section{Related Work}
A common pervasive theme in DG literature is domain-invariant representation learning, which is based on the idea of aligning feature distributions among different source domains, with the hope that the learned  invariance can be generalized to target domains. For instance, ~\cite{li2018mmd} achieved distribution alignment in the latent space of an autoencoder by using adversarial learning and the maximum mean discrepancy criteria. ~\cite{li2018cdann} matched conditional feature distributions across domains, enabling alignment of multimodal distributions for all class labels. ~\cite{liu2021ijcai21} exploited both the conditional and label shifts, and proposed a Bayesian variational inference framework with posterior alignment to reduce both the shifts simultaneously. However, existing works overemphasize the importance of joint distribution alignment which might hurt class discriminative information. Different from them, we propose a novel angular invariance as well as the accompanied norm shift assumption, and develop a learning framework based on the proposed term of invariance.

Meta-learning was introduced into the DG community by~\cite{li2018mldg} and has drawn increasing attention. The main idea is to divide the source domains into meta-train-domains and meta-test-domain to simulate domain shift, and regulate the model trained on meta-train-domains to perform well on meta-test-domain. Data augmentation has also been exploited for DG, which augments the source data to increase the diversity of training data distribution. For instance, ~\cite{wang2020mixup} employed the mixup~\cite{zhang2018mixup_origin} technique across multiple domains and trained model on the augmented heterogeneous mixup distribution, which implicitly enhanced invariance to domain shifts.

Different from the above DG methods which focus on training phase, test-time adaptation is a class of methods focusing on test phase, i.e., adjusting the model using online unlabeled data and correcting its prediction by itself during test time. ~\cite{wang2020Tent} proposed fully test-time adaptation, which modulates the BN parameters by minimizing the prediction entropy using stochastic gradient descent. ~\cite{iwasawa2021T3A} proposed a test-time classifier adjustment module for DG, which updates pseudo-prototypes for each class using online unlabeled data augmented by the base classifier trained on the source domains. We empirically show that AIDGN can effectively make the decision boundaries of all categories separate from each other and reduce the uncertainty of predictions, so that the existing test-time adaptation methods based on entropy minimization is not necessary.

We also show that our proposed AIDGN theoretically justifies and generalizes the recent proposed MAG loss for face recognition~\cite{meng2021magface}.

\section{Methodology}
In this section, we first formulate the DG problem. Secondly, we explain the proposed angular invariance and norm shift assumption. Lastly, we introduce our angular invariance domain generalization network (AIDGN). 
(\textit{Proofs for this section can be found in Appendix A of the supplementary material.})

\subsection{Problem Formulation}
Give $N$ source domains $\{\mathcal{P}_{\mathcal{X \times Y}}^{d}\}_{d=1}^N$ subject to $\mathcal{P}_{\mathcal{X \times Y}}^{d}\neq \mathcal{P}_{\mathcal{X \times Y}}^{d'}$ for $\{d,d'\} \subset \{1,2,\ldots,N\}$, and a target domain $\mathcal{P}_{\mathcal{X \times Y}}^{t}$ on the input-output space $\mathcal{X\times Y}$. DG tasks assume $\mathcal{P}_{\mathcal{X \times Y}}^{d}\neq \mathcal{P}_{\mathcal{X \times Y}}^{t}$ for $d=1,\ldots,N$ and focus on $C$-class single label classification tasks. 
Let $\mathcal{H}=\{h_{\theta}|\theta \in \Theta\}$ be a hypothesis space parametrized by $\theta \in \Theta$. 
For $d=1,\ldots,N$, there are $n_d$ independently identically distributed instances $\lbrace (\textbf{x}_i^d, y_i^d)\rbrace_{i=1}^{n_d}$ sampled from the $d$-th source domain $\mathcal{P}_{\mathcal{X \times Y}}^{d}$.
The goal of DG is to output a hypothesis $\hat{h}\in\mathcal{H}$ such that the target risk is minimized for a given loss $\ell\left(h(\cdot),\cdot\right)$, i.e.,
\begin{equation}\label{eq:goal}
    \hat{h}=\arg \min_{h\in\mathcal{H}} \mathbb{E}_{\mathcal{P}_{\mathcal{X \times Y}}^{t}}\left[\ell(h(\textbf{X}),Y)\right].
\end{equation}

\subsection{Angular Invariance and Norm Shift}
Celebrated for capturing empirical universal visual features, convolutional neural networks (CNNs) pre-trained on the ImageNet dataset~\cite{deng2009imagenet} have been adopted by a wide range of visual tasks. To take full advantage of a pre-trained CNN $\pi$, we regard $\pi$ as a feature extractor from the original input space $\mathcal{X}$ to a latent representation space $\mathcal{Z}$. Then a hypothesis $h$ comprises a feature extractor $\pi$ and a classifier $f$, i.e., $h=f\circ\pi$.

Studies have shown that each dimension of a CNN $\pi$ output is capturing some abstract concepts (e.g., eye, tail)~\cite{zeiler2014visualizing}
. Considering the relationship among concepts of the same class objects in the real-world is stable, the angular invariance and norm shift assumption says that the $\pi$-mapped feature of different domains are invariant in the angular coordinates, but varies in the norm coordinate. For simplicity, we introduce a random variable $D$ indexing the $d$-th source domain if $D=d$. The proposed assumption states as follows.

\begin{assumption}[angular invariance]\label{assump:angularInvariance}
Suppose the marginal distributions $\{\mathcal{P}_{\mathcal{X}}^{d}\}_{d=1}^N\cup\{\mathcal{P}_{\mathcal{X}}^t\}$ on the input space $\mathcal{X}$ are continuous. Let $\pi:\mathcal{X}\rightarrow \mathcal{Z}\subset \mathbb{R}^n$ be a feature extraction mapping such that the $\pi$-push forward probability density funcitons (p.d.f.s) $\{p^{d}(\textbf{z})\}_{d=1}^N\cup\{p^t(\textbf{z})\}$ exist in the latent space $\mathcal{Z}$. Let $(r,\phi_1,\ldots,\phi_{n-1})=g(z_1,\ldots,z_n)$ be the polar reparametrization \cite{blumenson1960derivation} of the Cartesian coordinates $\textbf{z}=(z_1,\ldots,z_n)$. The angular invariance assumption for DG is quantified by the equations: Let $\boldsymbol{\phi}=(\phi_1,\ldots,\phi_{n-1})$, for $d=1,\ldots,N$, 
\begin{equation}\label{eq:angular invariance}
        p(\boldsymbol{\phi}|Y,D=d)=p^t(\boldsymbol{\phi}|Y)
\end{equation}
\end{assumption}

The polar reparametrization $g(\cdot)$ is bijective and $p(r,\boldsymbol{\phi}|Y,D)=p(\boldsymbol{\phi}|Y,D)p(r|\boldsymbol{\phi},Y,D)$, therefore the difference between the target conditional p.d.f. (c.p.d.f.) $p^t(\textbf{z}|Y)$ and the $d$-th source c.p.d.f. $p(\textbf{z}|Y,d)$ is captured by the difference between the norm c.p.d.f.s $p^t(r|\boldsymbol{\phi},Y)$ and $p(r|\boldsymbol{\phi},d,Y)$.


\begin{theorem}\label{theorem:Reduce}
Suppose $support(p^t(\textbf{z}))\subset support(p(\textbf{z}|D))$. If the angular invariance assumption~\ref{assump:angularInvariance} holds, then for $d=1,\ldots,N$, $p^t(\textbf{z}|Y)/p(\textbf{z}|D=d,Y)$ exists and satisfies
\begin{equation}\label{eq:ratioweight}
 \frac{p^t(\textbf{z}|Y)}{p(\textbf{z}|d,Y)}=\frac{ p^t(r|\boldsymbol{\phi},Y)}{p(r|\boldsymbol{\phi},d,Y)}\triangleq w(r|\boldsymbol{\phi},d,y).
\end{equation}
\end{theorem} 

The theorem~\ref{theorem:Reduce} says that under the angular invariance assumption, we may reduce the degrees of freedom of comparing target and source c.p.d.f.s from $n$ to $1$. However, the aporia of DG is that no target instances could be observed during training. Thus an additional assumption is essential to overcome the zero-sample dilemma. Following the maximum entropy principle~\cite{jaynes1957information}, we adopt the following distributional assumptions on the conditional target and source norms.

\begin{assumption}[maximum entropy norm distribution]\label{assump:maximumEntropy} Conditioned on $Y=y$ and $\boldsymbol{\phi}$, (I) The target norm in space $\mathcal{Z}$ follows a continuous uniform distribution\footnote{The uniform distribution is the maximum (differential) entropy distribution for a continuous random variable with a fixed range.} $Uni[\alpha_{y,\boldsymbol{\phi}},\beta_{y,\boldsymbol{\phi}}]$ with $\delta_{y,\boldsymbol{\phi}}=\beta_{y,\boldsymbol{\phi}}-\alpha_{y,\boldsymbol{\phi}}>0$, i.e., $p^t(r|y,\boldsymbol{\phi};\alpha_{y,\boldsymbol{\phi}},\beta_{y,\boldsymbol{\phi}})=1/\delta_{y,\boldsymbol{\phi}}$. (II) The $d$-th source domain norm in space $\mathcal{Z}$ follows an exponential distribution\footnote{The exponential distribution is the maximum (differential) entropy distribution with positive support and a fixed expectation.} $Exp[1/\mu_{d,y,\boldsymbol{\phi}}],\mu_{d,y,\boldsymbol{\phi}}>0$, i.e., $p(r|d,y;\mu_{d,y,\boldsymbol{\phi}})=1/\mu_{d,y,\boldsymbol{\phi}} exp(r/\mu_{d,y,\boldsymbol{\phi}})$.
\end{assumption}

With the angular invariance and the maximum entropy assumption, we can compare $p^t(\textbf{z}|y)$ and $p(\textbf{z}|d,y)$ analytically.
\begin{corollary}\label{cor:RatioanalyticForm}
When assumption~\ref{assump:angularInvariance} and assumption~\ref{assump:maximumEntropy} hold,
\begin{equation}\label{eq:RatioanalyticForm}
    w(r|\boldsymbol{\phi},d,y)
    =\frac{\mu_{d,y} exp(\frac{r}{\mu_{d,y,\boldsymbol{\phi}}})}{\delta_{y,\boldsymbol{\phi}}} \approx \frac{\mu_{d,y,\boldsymbol{\phi}}+r}{\delta_{y,\boldsymbol{\phi}}}.
\end{equation}
\end{corollary}

Recalling that DG aims to learn classifiers, next we consider the behavior when $Y$ varies, i.e., $p^t(y|\textbf{z})$ and $p(y|\textbf{z},D)$.

\subsection{The AIDGN Method}
Before formally introducing the proposed AIDGN method, We discuss the motivation of adopting the von-Mises Fisher (vMF) mixture model. Specifically, we inspect $p(\textbf{z}|Y)$ and $p(\boldsymbol{\phi}|Y)$, where $\boldsymbol{\phi}=(\phi_1,\ldots,\phi_{n-1})$ is the angular coordinates after a polar reparameterization of $\textbf{z}$. By the law of total probability, the source c.p.d.f. $p^s(\textbf{z}|Y)$ decomposes as
\begin{equation}\label{eq:totalSum}
    \begin{aligned}
    p^s(\textbf{z}|Y)&=\sum_{d=1}^N p(\textbf{z}|d,Y)p(d|Y)\propto \sum_{d=1}^N p(r,\boldsymbol{\phi}|d,Y)p(d|Y)
    \end{aligned}
\end{equation}
When the angular invariance and norm shift assumption~\ref{assump:angularInvariance} holds, the factors $p(\textbf{z}|d,Y)$ and $p(d|Y)$ might varies w.r.t. the domain index $d$. Therefore modeling $p^s(\textbf{z}|Y)$ urges the modeling of $p(\textbf{z}|d,Y)$ w.r.t. $d=1,\ldots,N$. In sharp contrast, the angular invariance guarantees the modeling of the source c.p.d.f. $p^s(\boldsymbol{\phi}|Y)$ is as easy as any $p(\boldsymbol{\phi}|Y,d)$ for $d=1,\ldots,N$. By eq.~\eqref{eq:angular invariance}, $p^s(\boldsymbol{\phi}|Y)=\sum_{d=1}^N p(\boldsymbol{\phi}|Y,d)p(d|Y)=p(\boldsymbol{\phi}|Y,d)\sum_{d=1}^N p(d|Y)=p(\boldsymbol{\phi}|Y,d)$.
Therefore, the much simpler assumption choice is a model related to $p^s(\boldsymbol{\phi})$.
Notice that the angular coordinates of the latent representation $\textbf{z}$ are invariant to the $L_2$ normalization $\mathcal{G}(\textbf{z})$, i.e., $\mathcal{G}\circ g: (r,\boldsymbol{\phi})\mapsto (1,\boldsymbol{\phi})$, where $g$ is the polar reparameterization and $\mathcal{G}(\textbf{z})$ is
\begin{equation}\label{eq:l2normalization}
    \mathcal{G}(\textbf{z})=\textbf{z}/\sqrt{z_1^2+z_2^2+\ldots+z_n^2}.
\end{equation}
The formulation of the proposed AIDGN begins with the vMF mixture assumption on the $L_2$ normalized $\mathcal{G}(\textbf{Z})$.

\begin{assumption}[von-Mises Fisher mixture]\label{assump:vMFMixture}
Suppose that the assumption~\ref{assump:angularInvariance} is satisfied, let $\textbf{Z}^*\triangleq \mathcal{G}(\textbf{Z})$ be the $L_2$ normalization of the latent representation $\textbf{Z}=\pi(\textbf{X})$. In the $C$-category DG classification setting, $\textbf{Z}^*$ is assumed to follow a von-Mises Fisher mixture distribution,
\begin{equation}\label{vMFmixture}
    p(\textbf{z}^*)=\sum_{y=1}^C p(Y=y)p(\textbf{z}^*|y)=\sum_{y=1}^C p(y)\mathcal{V}(\textbf{z}^*;\textbf{w}_y,\kappa),
\end{equation}
such that the posterior p.d.f. $p(y|\textbf{z},D)$ is invariant to the $L_2$ normalized posterior $p(y|\textbf{z}^*)$ 
\begin{equation}
    p(y|\textbf{z},D)=p(y|\textbf{z}^*).
\end{equation}
The $y$-th component of the mixture $p(\textbf{z}^*|y)=\mathcal{V}(\textbf{z}^*;\textbf{w}_y,\kappa)$ is the p.d.f. of a vMF distribution,
\begin{equation}\label{eq:vMFdistribution}
\mathcal{V}(\textbf{z}^*;\textbf{w}_y,\kappa)=\frac{\kappa^{n/2-1}}{(2\pi)^{n/2}I_{n/2-1}(\kappa)}exp(\kappa \textbf{w}_y^{\top}\textbf{z}^*),
\end{equation}
where $I_n$ denotes the first kind Bessel function at order $n$.
\end{assumption}

Within the vMF mixture model, the angular invariance induces a relationship between the target posterior $p^t(y|\textbf{z})$ and the source mixture components $p(\textbf{z}^*|y)$. 
\begin{theorem}\label{theorem:posterior}
When the assumptions~\ref{assump:angularInvariance} and~\ref{assump:vMFMixture} are satisfied, then 

\begin{small}
\begin{equation}\label{eq:posterior}
\begin{aligned}
        p^t(y|\textbf{z})&=\frac{p(\textbf{z}^*|Y=y)w(r|\boldsymbol{\phi},d,y)P(Y=y)}{\sum_{c=1}^C p(\textbf{z}^*|Y=c)w(r|\boldsymbol{\phi},d,c)P(Y=c) }\\
        &=\frac{exp(\kappa \textbf{w}_y^{\top}\textbf{z}^*)w(r|\boldsymbol{\phi},d,y)P(Y=y)}{\sum_{c=1}^C exp(\kappa \textbf{w}_c^{\top}\textbf{z}^*)w(r|\boldsymbol{\phi},d,c)P(Y=c) }.
\end{aligned}
\end{equation}
\end{small}
\end{theorem}

The eq.~\eqref{eq:posterior} in theorem \ref{theorem:posterior} promises an optimization objective when there are enough observations for each source domain. However, the sample size is often  prohibitive. Even if the sample complexity could be satisfied by the sample size, inputting too many source instances in mini-batches is not practical for a DNN. On the other hand, the empirical estimation $\hat{w}(r|\boldsymbol{\phi},d,c)=\hat{p}^t(\textbf{z}|\boldsymbol{\phi},c)/\hat{p}(\textbf{z}|\boldsymbol{\phi},d,c)$ goes to infinity when an instance $\textbf{z}$ with polar coordinates $(\boldsymbol{\phi},r)$ is not observed in the $c$-th class and $d$-th domain ($\hat{p}(\textbf{z}|\boldsymbol{\phi},d,c)=0$). 

The above practical concerns motivate us to modify eq.~\eqref{eq:posterior}. For ease of illustration, we rewrite the vMF exponent factors $\textbf{w}_c^{\top}\textbf{z}^*\triangleq cos(\theta_c), c=1,\ldots,C$. Considering the relative modification effect of the magnitude of observed empirical estimates $\hat{w}(r|\boldsymbol{\phi},d,y)$ on $cos(\theta_y)$, a finite $\hat{w}(r|\boldsymbol{\phi},d,y)$ endows a relatively (compared to an infinite empirical estimate $\hat{w}(r|\boldsymbol{\phi},d,c)$) small weight of the corresponding $cos(\theta_y)$, and thus a small credibility on $\theta_y$. The proposed AIDGN reinterprets such credibility on $\theta_y$ as a perturbation to $\theta_y$ such that $exp(\kappa cos(\theta_y+\gamma w(r|\boldsymbol{\phi},d,y)))\approx exp(\kappa cos(\theta_y))w(r|\boldsymbol{\phi},d,y)$. For instances living on the mass manifold that are not observed, AIDGN reinterprets the infinite empirical estimate $\hat{w}(r|\boldsymbol{\phi},d,c)$ as absolute confidence on the corresponding $\theta_c$, and no perturbation is added. 
In the light of reinterpreting $\hat{w}(r|\boldsymbol{\phi},d,\cdot)$ as perturbation on $\theta$, enforcing $p^t(y|\textbf{z})$ to be close to $1$ reduces to enforcing 
\begin{equation}
    cos(\theta_y+\gamma w(r|\boldsymbol{\phi},d,y))p(y)\geq cos(\theta_c)p(c), 
\end{equation}
where $c\neq y$ is indexing any wrong class of $(\textbf{z},y)$. To derive the optimization objective of AIDGN, we adopt the maximum entropy principle again on the label distribution.

\begin{theorem}\label{theorem:Loss}
Let $\Delta$ be a $C-1$ simplex and let $\textbf{P}=(P_1,\ldots,P_C) \in \Delta$ denote a distribution of classes, then

\begin{small}
\begin{equation}\label{eq:labelMEP}
\begin{aligned}
    \textbf{P}^*&=(\frac{e^{\kappa(cos(\theta_1)-cos(\theta_y+\gamma w(r|\boldsymbol{\phi},d,y)))}}{\sum_{c=1}^C e^{\kappa(cos(\theta_c)-cos(\theta_y+\gamma w(r|\boldsymbol{\phi},d,y)))}},\\
    &\ldots,\frac{e^{\kappa(cos(\theta_C)-cos(\theta_y+\gamma w(r|\boldsymbol{\phi},d,y)))}}{\sum_{c=1}^C e^{\kappa(cos(\theta_c)-cos(\theta_y+\gamma w(r|\boldsymbol{\phi},d,y)))}})\\
    =&\arg \max_{\textbf{P}\in \Delta} \kappa [ \sum_{c\neq y}^C P_c( cos(\theta_c)-cos(\theta_y+\gamma w(r|\boldsymbol{\phi},d,y)))\\
    &+P_y(cos(\theta_y+\gamma w(r|\boldsymbol{\phi},d,y)
     -cos(\theta_y+\gamma w(r|\boldsymbol{\phi},d,y)))]\\
     &+\sum_{c=1}^C P_c(-log P_c).
\end{aligned}
\end{equation}
\end{small}
At $\textbf{P}=\textbf{P}^*$, the maximum of the right hand side (r.h.s.) is
\begin{equation}\label{eq:AIDGNloss}
\begin{aligned}
    &\ell_{AIDGN}(\textbf{z},y)\\
    &=-log \frac{e^{\kappa cos(\theta_y+\gamma w(r|\boldsymbol{\phi},d,y))}}{e^{\kappa cos(\theta_y+\gamma w(r|\boldsymbol{\phi},d,y))}+\sum_{c\neq y}e^{\kappa cos(\theta_y)}}.
\end{aligned}
\end{equation}
\end{theorem}

In theorem~\ref{theorem:Loss}  eq.~\eqref{eq:AIDGNloss}, we derived the loss $\ell_{AIDGN}(\textbf{z},y)$. We next inspect the gradient behavior of the loss. We take the gradients of $\ell_{AIDGN}(\textbf{z},y)$ w.r.t. the correct center $\textbf{w}_y$ and w.r.t. an incorrect center $\textbf{w}_c$. At the $k$-th time step of training, let $\textbf{p}_k$ be the prediction vector of the model, 
\begin{equation}\label{eq:gradients}
\begin{cases}
    \partial \ell_{AIDGN}(\textbf{z},y)/\partial \textbf{w}_y \propto \kappa \left(1-p_k(y) \right) \textbf{z},\\
    \partial \ell_{AIDGN}(\textbf{z},y)/\partial \textbf{w}_c \propto \kappa p_k(c)\textbf{z}.
\end{cases}
\end{equation}

When performing the gradient descent training, the eq.\eqref{eq:gradients} reveals the ideal behavior of decreasing gap between correct prediction probability and 1, and a non-ideal behavior in our DG formulation, i.e., shrinking the latent representation $\textbf{z}$. Specifically, when we assume the angular invariance, the norm of $\textbf{z}$ carries helpful information for comparing $p^t(\textbf{z}|Y)$ and $p(\textbf{z}|D,Y)$ as discussed in theorem~\ref{theorem:Reduce}, and the shrinking of $\textbf{z}$ causes information loss in the norm coordinate.

To combat against the shrinking tendency brought by the $\ell_{AIDGN}(\textbf{z},y)$ loss, we introduce an information regularizer to penalize the loss of information. Concretely, we introduce an ideal norm distribution (which can be regarded as the norm distribution of an ideal source domain) that follows an exponential distribution $Exp(1/\mu^*)$. The information loss regularizer $KL(\cdot \| \cdot)$ is the Kullback-Leibler (K-L) divergence 
that calculates the relative entropy between $Exp(1/\mu^*)$ and $Exp(1/\mu)$.
\begin{equation}
  KL(1/\mu^* \|1/\mu)=log(\mu/\mu^*)+\mu^*/\mu-1
\end{equation}

Finally, the final optimization objective of the proposed AIDGN is 
\begin{equation}\label{eq:finalOO}
\begin{aligned}
        &L_{AIDGN}= \sum_{d=1}^N  \sum_{i=1}^{n_d} \eta KL(1/\mu^* \|1/\mu_{d,y,\boldsymbol{\phi}})\\
        &-log \frac{e^{\kappa cos(\theta_y+\gamma w(r|\boldsymbol{\phi},d,y))}}{e^{\kappa cos(\theta_y+\gamma w(r|\boldsymbol{\phi},d,y))}+\sum_{c\neq y}e^{\kappa cos(\theta_y)}},
\end{aligned}
\end{equation}
where $\eta>0$ is a hyperparameter controlling the trade-off between information memorizing and forgetting.

\begin{table*}[th]
\centering
\scalebox{0.91}{
\begin{tabular}{lccccc}
\toprule
 Method & PACS & VLCS & OfficeHome & TerraIncognita & Avg \\
\midrule
ERM~\cite{vapnik1999erm} & 85.5 $\pm$ 0.2 & 77.5 $\pm$ 0.4 & 66.5 $\pm$ 0.3 & 46.1 $\pm$ 1.8 & 68.9 \\
ERM$^\dagger$~\cite{vapnik1999erm} & 85.8 $\pm$ 0.2 & 77.7 $\pm$ 0.4 & 66.9 $\pm$ 0.2 & 45.8 $\pm$ 1.5 & 69.1 \\
IRM~\cite{arjovsky2019irm} & 83.5 $\pm$ 0.8 & 78.5 $\pm$ 0.5 & 64.3 $\pm$ 2.2 & 47.6 $\pm$ 0.8 & 68.5 \\
DRO~\cite{sagawa2019dro} & 84.4 $\pm$ 0.8 & 76.7 $\pm$ 0.6 & 66.0 $\pm$ 0.7 & 43.2 $\pm$ 1.1 & 67.6 \\
Mixup~\cite{wang2020mixup} & 84.6 $\pm$ 0.6 & 77.4 $\pm$ 0.6 & 68.1 $\pm$ 0.3 & 47.9 $\pm$ 0.8 & 69.5 \\
MLDG~\cite{li2018mldg} & 84.9 $\pm$ 1.0 & 77.2 $\pm$ 0.4 & 66.8 $\pm$ 0.6 & 47.7 $\pm$ 0.9 & 69.2 \\
CORAL~\cite{sun2016coral} & 86.2 $\pm$ 0.3 & 78.8 $\pm$ 0.6 & 68.7 $\pm$ 0.3 & 47.6 $\pm$ 1.0 & 70.3 \\
MMD~\cite{li2018mmd} & 84.6 $\pm$ 0.5 & 77.5 $\pm$ 0.9 & 66.3 $\pm$ 0.1 & 42.2 $\pm$ 1.6 & 67.7 \\
DANN~\cite{ganin2016dann} & 83.6 $\pm$ 0.4 & 78.6 $\pm$ 0.4 & 65.9 $\pm$ 0.6 & 46.7 $\pm$ 0.5 & 68.7 \\
CDANN~\cite{li2018cdann} & 82.6 $\pm$ 0.9 & 77.5 $\pm$ 0.1 & 65.8 $\pm$ 1.3 & 45.8 $\pm$ 1.6 & 67.9 \\
MTL~\cite{blanchard2021MTL} & 84.6 $\pm$ 0.5 & 77.2 $\pm$ 0.4 & 66.4 $\pm$ 0.5 & 45.6 $\pm$ 1.2 & 68.5 \\
SagNet~\cite{nam2021sagnet} & 86.3 $\pm$ 0.2 & 77.8 $\pm$ 0.5 & 68.1 $\pm$ 0.1 & 48.6 $\pm$ 1.0 & 70.2 \\
ARM~\cite{zhang2020arm} & 85.1 $\pm$ 0.4 & 77.6 $\pm$ 0.3 & 64.8 $\pm$ 0.3 & 45.5 $\pm$ 0.3 & 68.3 \\
VREx~\cite{krueger2021vrex} & 84.9 $\pm$ 0.6 & 78.3 $\pm$ 0.2 & 66.4 $\pm$ 0.6 & 46.4 $\pm$ 0.6 & 69.0 \\
RSC~\cite{huang2020rsc} & 85.2 $\pm$ 0.9 & 77.1 $\pm$ 0.5 & 65.5 $\pm$ 0.9 & 46.6 $\pm$ 1.0 & 68.6 \\
\hline
AIDGN (ours) & {\bf86.6 $\pm$ 0.3} & {\bf78.9 $\pm$ 0.3} & {\bf68.8 $\pm$ 0.2} & {\bf49.4 $\pm$ 0.6} & {\bf70.9} \\

\bottomrule
\end{tabular}}
\caption{\textbf{Benchmark Comparisons.} Out-of-domain classification accuracies(\%) on PACS, VLCS, OfficeHome and TerraIncognita are shown. Note that the results of ERM$^\dagger$ are reproduced by us, and other numbers are from DomainBed.}
\label{benchmark comparisons}
\end{table*}

\begin{remark}
The proposed AIDGN optimization objective eq.~\eqref{eq:finalOO} theoretically justifies and generalizes the recent proposed MAG loss for face recognition~\cite{meng2021magface}. When there is only one source domain, it can be easily verified that with a first-order Maclaurin's expansion approximation w.r.t. the log term in the regularizer $KL(1/\mu^* \|1/\mu)$, the proposed AIDGN loss eq.~\eqref{eq:finalOO} degenerates to the MAG loss.
\end{remark}

\section{Experiments}

\subsection{Experimental Settings}

\paragraph{Datasets.} 
We conduct our experiments\footnote{Codes are avalable at \href{https://github.com/JinYujie99/aidgn}{https://github.com/JinYujie99/aidgn}} on four public benchmark datasets to evaluate the effectiveness of the proposed AIDGN. PACS~\cite{li2017pacs} comprises four domains $d\in \lbrace$photo, art, cartoon, sketch$\rbrace$, containing 9991 images of 7 categories. VLCS~\cite{fang2013vlcs} comprises four photographic domains $d\in \lbrace$VOC2007, LabelMe, Caltech101, SUN09$\rbrace$, with 10729 samples of 5 classes. OfficeHome~\cite{venkateswara2017office} has four domains $d\in \lbrace$art, clipart, product, real$\rbrace$, containing 15500 images with a larger label sets of 65 categories. TerraIncognita~\cite{beery2018terra} comprises photos of wild animals taken by cameras at different locations. Following~\cite{gulrajani2020domainbed}, we use domains of  $d\in \lbrace$L100, L38, L43, L46$\rbrace$, which include 24778 samples and 10 classes.

\paragraph{Evaluation Protocol.} For a fair comparison, we use the  DomainBed training-domain validation set protocol~\cite{gulrajani2020domainbed} for model selection. For training, we randomly split each training domain into 8:2 training/validation splits, 
choose the model 
on the overall validation set, and then evaluate its performance on the target domain set. We report the mean and standard deviation of out-of-domain classification accuracy from three different runs with different training-validation splits.

\paragraph{Implementation Details.} For all datasets, we use ResNet-50~\cite{he2016resnet} pre-trained on ImageNet~\cite{deng2009imagenet} as the feature extractor $\pi$ and one fully connected layer as the classifier $f$.  We construct a mini-batch containing all source domains where each domain has 32 images. We freeze all the batch normalization (BN) layers from pre-trained ResNet since different domains in a mini-batch follow different distributions.  The network is trained for 5000 iterations using the Adam~\cite{kingma2015adam} optimizer. To estimate $\delta_{y,\boldsymbol{\phi}}$ and $\mu_{d,y,\boldsymbol{\phi}}$ in AIDGN loss~\eqref{eq:finalOO}, we treat all $\delta_{y,\boldsymbol{\phi}}$ as one single hyperparameter $\delta$, and perform in-batch estimation for the $d$-th domain norm scale parameter, i.e., $\mu_d$, while ignoring the index $y$ and $\boldsymbol{\phi}$. Specifically, we estimate $\mu_d$ by computing the average norm of samples of the $d$-th domain in a minibatch. We do this mainly for two reasons. First, according to the norm shift assumption, domain index $d$ is more relevant to the norm distribution. Secondly, since the feasible label sets and angular sets can be extensive, it is almost impossible to estimate $\mu_{d,y,\boldsymbol{\phi}}$ in a minibatch precisely. Moreover, for the log term in the KL regularizer, we approximate it with a first-order Maclaurin's expansion to make the overall objective function convex.  Following~\cite{gulrajani2020domainbed}, we conduct a random search over the joint distribution of hyperparameters. More implementation details about data preprocessing techniques, model architectures, hyperparameters, and experimental environments can be found in the supplementary material.

\begin{table*}[ht]
\centering
\scalebox{0.92}{
\begin{tabular}{lccccc}
\toprule
  & C & L & V & S & Avg \\
\midrule
AIDGN w/o RD & 96.9 $\pm$ 0.3 & 63.5 $\pm$ 0.5 & 73.0 $\pm$ 0.5 & 75.8 $\pm$ 0.6 & 77.3 \\
AIDGN w/o R & 97.3 $\pm$ 0.5 & 64.9 $\pm$ 0.8 & {\bf73.4 $\pm$ 0.4} & 77.9 $\pm$ 1.0 & 78.4  \\
AIDGN w/o D & 97.6 $\pm$ 0.3 & 64.7 $\pm$ 0.4 & {\bf73.4 $\pm$ 0.4} & 78.1 $\pm$ 0.7 & 78.4  \\
AIDGN   & {\bf98.3 $\pm$ 0.2} & {\bf65.7 $\pm$ 0.4} & 73.1 $\pm$ 0.4 & {\bf78.7 $\pm$ 0.7} & {\bf78.9}  \\

\bottomrule
\end{tabular}}
\caption{\textbf{Ablation studies.} Out-of-domain classification accuracies(\%) with different optimization components on VLCS.}
\label{ablation studies}
\end{table*}

\subsection{Benchmark Comparisons}

We compare our proposed AIDGN with 14 available DG methods in  DomainBed~\cite{gulrajani2020domainbed}. Results on PACS, VLCS, OfficeHome, and TerraIncognita are reported in Table~\ref{benchmark comparisons}. (\textit{Details of each baseline and full results per dataset per domain can be found in the supplementary material.}) All the values of baselines are taken from DomainBed when ResNet-50 is used as a backbone network, except that ERM$^\dagger$ is reproduced by us. For all datasets, AIDGN achieves better average out-of-domain accuracy. Particularly, in TerraIncognita, the proposed method achieves 49.4\%, which is significantly better than the most competitive baseline, 48.6\%. Similarly, it is 86.6\% for PACS, 78.9\% for VLCS, and 68.8\% for OfficeHome, all of which outperform the previous best domain-invariant representation learning results. 

\subsection{Ablation Studies}

To verify the effectiveness of all components of the AIDGN optimization objective, we do ablation studies on the VLCS dataset. Specifically, we compare AIDGN with three variants: a) AIDGN w/o R: The model is trained without the KL regularizer.  b) AIDGN w/o D: The model is trained without estimation for the norm scale parameters of source domains, i.e., ignoring the $\mu_d$ term in the expression of $w(r|\boldsymbol{\phi},d,y)$. c) AIDGN w/o RD: The model is trained without both. The results are reported in Table~\ref{ablation studies}. Note that the results for all the three variants also follow the training-domain validation set model selection and hyperparameter random search protocol, and are obtained from three different runs with different dataset splits. The results show that the full model AIDGN outperforms all these variants, indicating that both the KL regularizer and the estimation for the source domain norm scale parameters are essential to our algorithm.

\subsection{Discussions on the AIDGN Latent Space}

\begin{figure}[t]
    \centering
    \subfigure[Ours(domain)]{
    \begin{minipage}{4.05cm}
    \centering
        \label{ours-domain}
        \includegraphics[width=4.05cm]{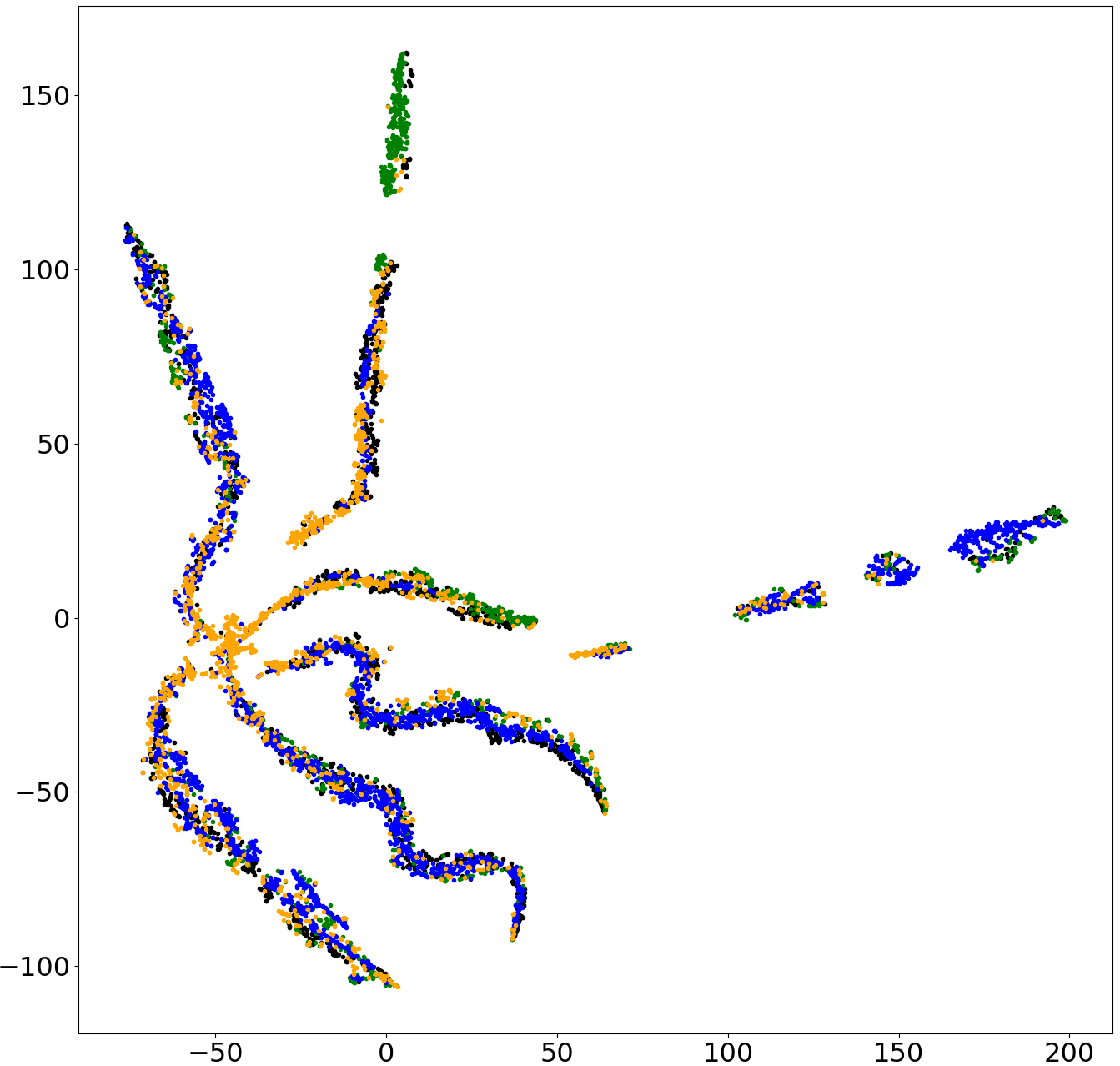}
    \end{minipage}
    }
    \subfigure[Ours(class)]{
    \begin{minipage}{4.05cm}
    \centering
        \label{ours-class}
        \includegraphics[width=4.05cm]{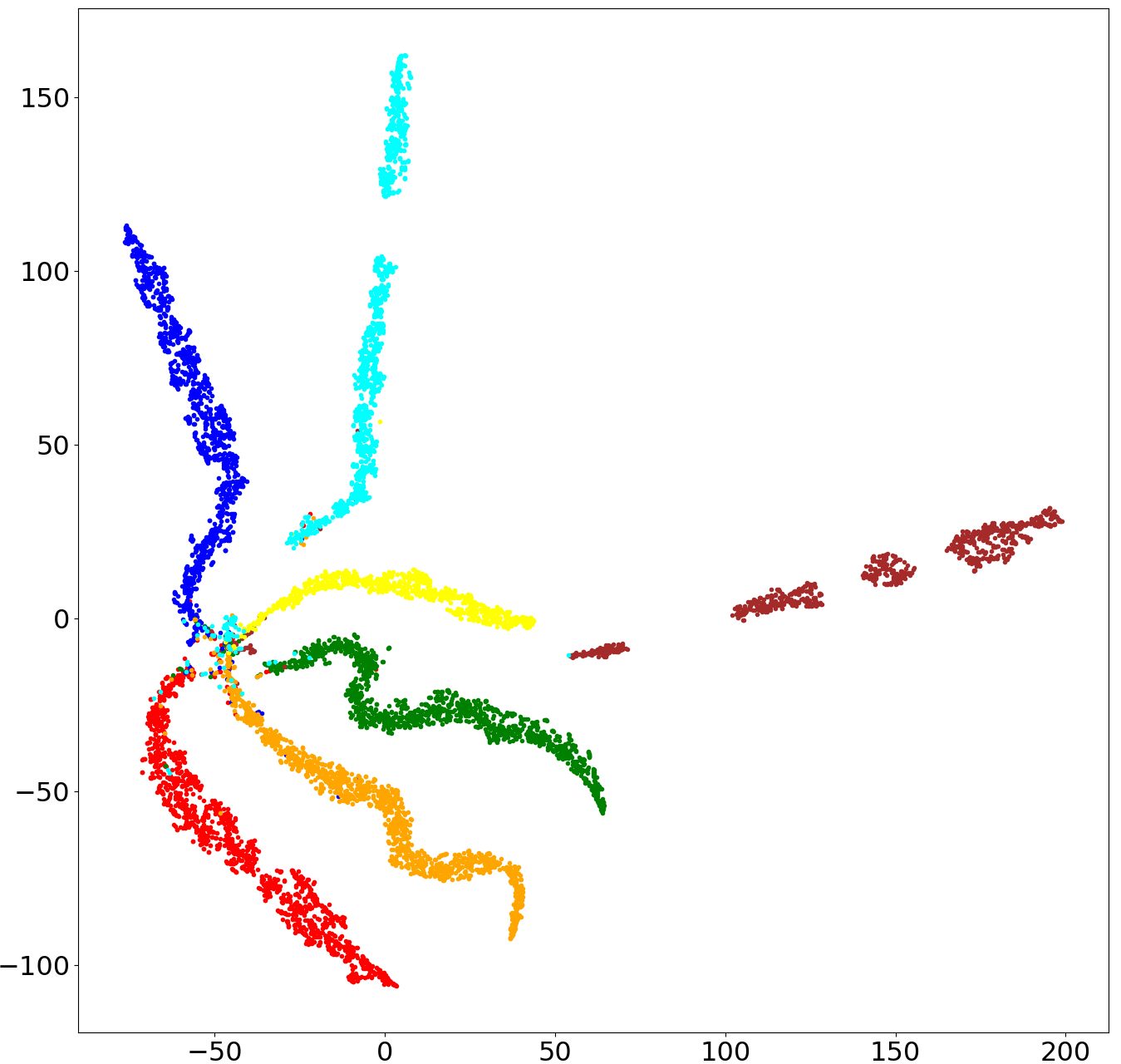}
    \end{minipage}
    }
\caption{Feature visualization for AIDGN on PACS: (a) different colors indicate  different domains, source domains include cartoon (black), photo (green) and sketch (blue) while the target domain is art-painting (orange); (b) different colors represent different classes. \textit{Best viewed in color(Zoom in for details).}}
\label{aidgn-visual}
\end{figure}

To analyze the learned latent space of AIDGN, we visualize the distributions of features with t-SNE in Fig~\ref{aidgn-visual}. It is shown that AIDGN can effectively make the decision boundaries of all categories separate from each other and better balance the intra-class compactness and the inter-class separation.

In addition, we test our model with existing test-time adaptation methods which are based on entropy minimization: T3A~\cite{iwasawa2021T3A} and Tent~\cite{wang2020Tent}. Since DomainBed~\cite{gulrajani2020domainbed} freezes the BN layers
, we test two slightly modified versions of Tent following~\cite{iwasawa2021T3A}. Specifically, Tent-BN adds one BN layer before the classifier and modulates its transformation and normalization parameters. Tent-C adapts the classifier to minimize prediction entropy. The average performances across PACS, VLCS, OfficeHome and TerraIncognita are shown in Fig~\ref{test-time adaptation}. (Please refer to the full per dataset and domain results in the supplementary material.) We can find that none of them can further improve the performance of AIDGN and our original AIDGN outperforms all the variants. While test-time adaptation modifies the model aiming at reducing prediction uncertainty caused mainly by domain shift, our proposed AIDGN models domain shift as norm parameters shift in the latent space, and is thus robust to the uncertainty brought by domain shift.

\begin{figure}
    \centering
    \includegraphics[width=8cm]{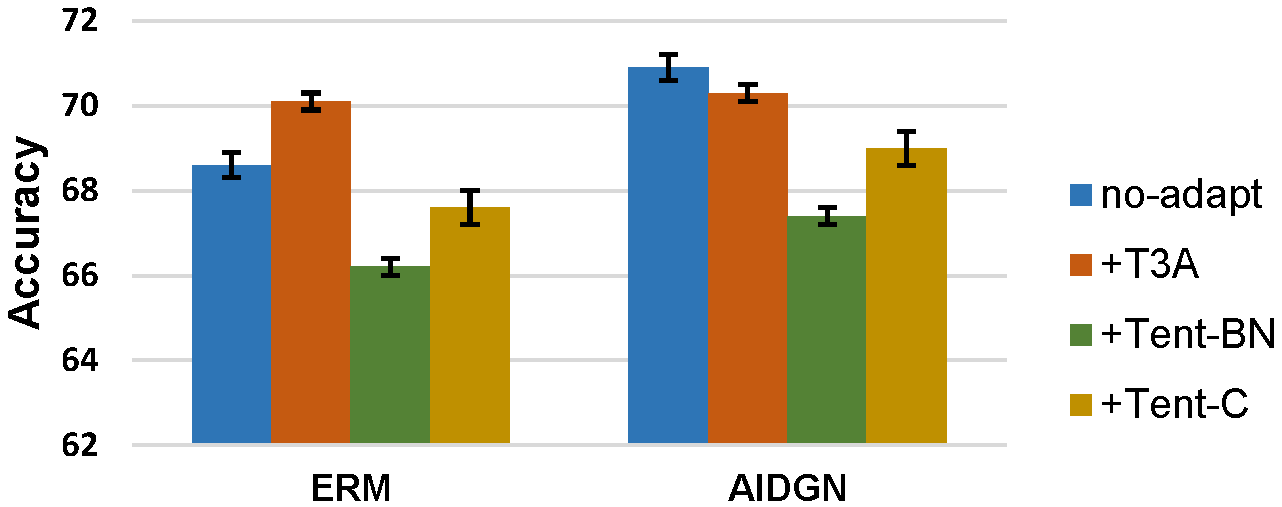}
    \caption{Average out-of-domain accuracies of ERM and AIDGN equipped with different test-time adaptation methods.}
    \label{test-time adaptation}
\end{figure}

\section{Conclusion}

In this paper, we introduce a novel angular invariance and norm shift assumption into domain generalization, inspired by the acknowledged fact that internal layers in convolutional neural networks capture high-level semantic concepts. We then propose a method based on the assumption and develop a practical optimization objective within a von-Mises Fisher mixture model. Extensive experiments on four benchmarks demonstrate the superior performance of our proposed method. While this work focuses on inter-domain invariance, it is complementary to ensemble learning which aims to learn robust classifiers. For future work, we will develop methods that take both invariance and robustness into consideration.

\section*{Acknowledgements}
This work is supported by the National Natural Science Foundation of China (No.62172011).
\small
\bibliographystyle{named}
\bibliography{ijcai22}
\normalsize
\onecolumn
\setcounter{equation}{0}
\setcounter{section}{0}
\setcounter{theorem}{0}
\setcounter{table}{0}
\setcounter{figure}{0}
\renewcommand\thesection{\Alph{section}}
\section{Proofs for AIDGN}
\begin{theorem}\label{theorem:reducingComplexity}
Suppose $support(p^t(\textbf{z}))\subset support(p(\textbf{z}|D))$. If the angular invariance assumption~\ref{assump:angularInvariance} holds, then for $d=1,\ldots,N$, $p^t(\textbf{z}|Y)/p(\textbf{z}|D=d,Y)$ exists and satisfies
\begin{equation}\label{eq:Reducing}
 \frac{p^t(\textbf{z}|Y)}{p(\textbf{z}|d,Y)}=\frac{ p^t(r|\boldsymbol{\phi},Y)}{p(r|\boldsymbol{\phi},d,Y)}\triangleq w(r|\boldsymbol{\phi},d,y).
\end{equation}
\end{theorem} 

\begin{proof} 
Suppose that $(r,\phi_1,\ldots,\phi_{n-1})=g(\textbf{z})$ is the polar transformation of the cartesian coordinates $\textbf{z}=(z_1,\ldots,z_n)$ in $\mathbb{R}$ space. The polar transformation $g(\cdot)$ is a bijective map with
\begin{equation}
\begin{aligned}
     \begin{cases}
    &r=g_1(z_1,\ldots,z_n)=\sqrt{z_1^2+z_2^2+\ldots+z_n^n},\\
    &\phi_1=g_2(z_1,\ldots,z_n)=arccot (z_1/\sqrt{z_1^2+z_2^2+\ldots+z_n^n}),\\
    &\phi_2=g_3(z_1,\ldots,z_n)=arccot (z_2/\sqrt{z_1^2+z_2^2+\ldots+z_n^n}),\\
    &\ \ \ \ \ \ \cdots,\\
    &\phi_{n-2}=g_{n-1}(z_1,\ldots,z_n)= arccot (z_{n-2}/\sqrt{z_1^2+z_2^2+\ldots+z_n^n}),\\
    &\phi_{n-1}=g_n(z_1,\ldots,z_n)=2 arccot (z_{n-1}+\sqrt{z_{n-1}^2+z_{n}^2}).
    \end{cases}   
\end{aligned}
\end{equation}
The inverse map $g^{-1}((r,\phi_1,\ldots,\phi_{n-1})=z_1,\ldots,z_n$ is
\begin{equation}
\begin{aligned}
     \begin{cases}
    &z_1=g^{-1}_1(r,\phi_1,\ldots,\phi_{n-1})=rcos(\phi_1),\\
    &z_2=g^{-1}_2(r,\phi_1,\ldots,\phi_{n-1})=rsin(\phi_1)cos(\phi_2),\\
    &z_3=g^{-1}_3(r,\phi_1,\ldots,\phi_{n-1})=rsin(\phi_1)sin(\phi_2)cos(\phi_3),\\
    &\ \ \ \ \ \ \cdots,\\
    &z_{n-1}=g^{-1}_{n-1}(r,\phi_1,\ldots,\phi_{n-1})=r(\prod_{i=1}^{n-2}sin(\phi_i))cos(\phi_{n-1}),\\
    &z_m=g^{-1}_n(r,\phi_1,\ldots,\phi_{n-1})=r(\prod_{i=1}^{n-1}sin(\phi_i)).
    \end{cases}   
\end{aligned}
\end{equation}
The Jacobian matrix of $g^{-1}$ is
\begin{equation}
    \textbf{J}_{g^{-1}}=\begin{pmatrix}
    cos(\phi_1) & -rsin(\phi_1) & 0 &\cdots & 0\\
    sin(\phi_1)cos(\phi_2) & rcos(\phi_1)cos(\phi_2) & -rsin(\phi_1)sin(\phi_2) &\cdots & 0\\
    \vdots &\vdots &\vdots &\ddots  &\vdots \\
    (\prod_{i=1}^{n-2}sin(\phi_i))cos(\phi_{n-1}) & \cdots &\cdots &\cdots & -r \prod_{i=1}^{n-1}sin(\phi_i)\\
    \prod_{i=1}^{n-1}sin(\phi_i) & r(\prod_{i=1}^{n-2}cos(\phi_i))sin(\phi_{n-1}) &\cdots &\cdots & r(\prod_{i=1}^{n-2}sin(\phi_i))cos(\phi_{n-1})
    \end{pmatrix}
\end{equation}
Thus the determinant of the Jacobian matrix $\textbf{J}_{g^{-1}}$, for any distribution in the space, is a function of polar coordinates $(r,\phi_1,\ldots,\phi_{n-1})$, i.e.,
 \begin{equation}\label{eq:independence}
  det(\textbf{J}_{g^{-1}})=\mathcal{F}(r,\phi_1,\ldots,\phi_{n-1}) . 
 \end{equation}
 \noindent From the p.d.f. transformation formula w.r.t. bijective maps, we have
\begin{equation}
\begin{aligned}
    \frac{p^t(z_1,\ldots,z_{n}|Y)}{p(z_1,\ldots,z_n|d,Y)}&=\frac{p^t(g(z_1,\ldots,z_n)|Y)det(\textbf{J}^t_{g})}{p(g(z_1,\ldots,z_n)|Y)det(\textbf{J}^d_{g})}=\frac{p^t(g(z_1,\ldots,z_n)|Y)1/det(\textbf{J}^t_{g^{-1}})}{p(g(z_1,\ldots,z_n)|Y)1/det(\textbf{J}^d_{g^{-1}})}\\
    &\overset{(a)}{=}\frac{p^t(g(z_1,\ldots,z_n)|Y)\mathcal{F}(r,\phi_1,\ldots,\phi_{n-1}))}{p(g(z_1,\ldots,z_n)|Y)\mathcal{F}(r,\phi_1,\ldots,\phi_{n-1}))}=\frac{p^t(r,\phi_1,\ldots,\phi_{n-1}|Y)}{p(r,\phi_1,\ldots,\phi_{n-1}|Y,d)}\\
    &=\frac{p^t(\boldsymbol{\phi}|Y)p^t(r|\boldsymbol{\phi},Y)}{p(\boldsymbol{\phi}|Y,d)p(r|\boldsymbol{\phi},d,Y)}\\
    &\overset{(b)}{=}\frac{p^t(r|\boldsymbol{\phi},Y)}{p(r|\boldsymbol{\phi},d,Y)},
\end{aligned}
\end{equation}
where (a) is invoking eq.~\eqref{eq:independence} and (b) is invoking the angular invariance assumption.

\end{proof}
\clearpage

\begin{theorem}
When the assumptions~\ref{assump:angularInvariance} and~\ref{assump:vMFMixture} is satisfied, then 
\begin{equation}
\begin{aligned}
        p^t(y|\textbf{z})&=\frac{p(\textbf{z}^*|Y=y)w(r|\boldsymbol{\phi},d,y)P(Y=y)}{\sum_{c=1}^C p(\textbf{z}^*|Y=c)w(r|\boldsymbol{\phi},d,c)P(Y=c) }
        =\frac{exp(\kappa \textbf{w}_y^{\top}\textbf{z}^*)w(r|\boldsymbol{\phi},d,y)P(Y=y)}{\sum_{c=1}^C exp(\kappa \textbf{w}_c^{\top}\textbf{z}^*)w(r|\boldsymbol{\phi},d,c)P(Y=c) }.
\end{aligned}
\end{equation}
\end{theorem}

\begin{proof}

From assumption 3, we assume that the distribution of the vMF components $p(\textbf{z}^*|y)=\mathcal{V}(\textbf{z}^*;\textbf{w}_y,\kappa)$ (composed of model parameters to be learned) satisfies $p(y|\textbf{z},D)=p(y|\textbf{z}^*)$, which is equivalent to
\begin{equation}\label{eq:posteriorinvariance}
    p(\textbf{z},y|D)/p(\textbf{z})=p(\textbf{z}^*,y)/p(\textbf{z}^*).
\end{equation}
Then we may complete the proof by invoking theorem~\ref{theorem:reducingComplexity} and some straightforward derivations,
\begin{equation*}
\begin{aligned}
          p^t(y|\textbf{z})&=\frac{p(Y=y)p^t(\textbf{z}|y)}{\sum_{c=1}^N p(Y=c)p^t(\textbf{z}|c)}\\
      &\overset{(a)}{=} \frac{p(Y=y)p(\textbf{z}|y,d)w(r|\boldsymbol{\phi},y,d)}{\sum_{c=1}^N p(Y=c)p(\textbf{z}|c,d)w(r|\boldsymbol{\phi},c,d)}\\
      &=\frac{[p(Y=y)/p(D=d)]p(\textbf{z}|y,d)w(r|\boldsymbol{\phi},y,d)}{\sum_{c=1}^N [p(Y=c)/p(D=d)]p(\textbf{z}|c,d)w(r|\boldsymbol{\phi},c,d)}\\
      &= \frac{p(y|d)p(\textbf{z}|y,d)w(r|\boldsymbol{\phi},y,d)}{\sum_{c=1}^N p(c|d)p(\textbf{z}|c,d)w(r|\boldsymbol{\phi},c,d)}\\
      &=\frac{p(\textbf{z},y|d)w(r|\boldsymbol{\phi},y,d)}{\sum_{c=1}^N p(\textbf{z},c|d)w(r|\boldsymbol{\phi},c,d)}\\
      &=\frac{[p(\textbf{z},y|d)/p(\textbf{z})]w(r|\boldsymbol{\phi},y,d)}{\sum_{c=1}^N [p(\textbf{z},c|d)/p(\textbf{z})]w(r|\boldsymbol{\phi},c,d)}\\
      &\overset{(b)}{=}\frac{[p(\textbf{z}^*,y)/p(\textbf{z}^*)]w(r|\boldsymbol{\phi},y,d)}{\sum_{c=1}^N [p(\textbf{z}^*,c)/p(\textbf{z}^*)]w(r|\boldsymbol{\phi},c,d)}\\
      &=\frac{p(\textbf{z}^*,y)w(r|\boldsymbol{\phi},y,d)}{\sum_{c=1}^N [p(\textbf{z}^*,c)w(r|\boldsymbol{\phi},c,d)}\\
      &=\frac{p(Y=y)p(\textbf{z}^*|y)w(r|\boldsymbol{\phi},y,d)}{\sum_{c=1}^N [p(Y=c)p(\textbf{z}^*|c)w(r|\boldsymbol{\phi},c,d)}\\
      &\overset{(c)}{=}\frac{exp(\kappa \textbf{w}_y^{\top}\textbf{z}^*)w(r|\boldsymbol{\phi},d,y)P(Y=y)}{\sum_{c=1}^C exp(\kappa \textbf{w}_c^{\top}\textbf{z}^*)w(r|\boldsymbol{\phi},d,c)P(Y=c) },
\end{aligned}
\end{equation*}
where we use theorem~\ref{theorem:reducingComplexity} at (a), eq.~\eqref{eq:posteriorinvariance} at (b), and the kernel of vMF components at (c).
\end{proof}
\clearpage

\begin{theorem}
Let $\Delta$ be a $C-1$ simplex and let $\textbf{P}=(P_1,\ldots,P_C) \in \Delta$ denote a distribution of classes, then
\begin{equation}\label{eq:labelMEP}
\begin{aligned}
    \textbf{P}^*&=(\frac{e^{\kappa(cos(\theta_1)-cos(\theta_y+\gamma w(r|\boldsymbol{\phi},d,y)))}}{\sum_{c=1}^C e^{\kappa(cos(\theta_c)-cos(\theta_y+\gamma w(r|\boldsymbol{\phi},d,y)))}},\ldots,\frac{e^{\kappa(cos(\theta_C)-cos(\theta_y+\gamma w(r|\boldsymbol{\phi},d,y)))}}{\sum_{c=1}^C e^{\kappa(cos(\theta_c)-cos(\theta_y+\gamma w(r|\boldsymbol{\phi},d,y)))}})\\
    =&\arg \max_{\textbf{P}\in \Delta} \kappa [ \sum_{c\neq y}^C P_c( cos(\theta_c)-cos(\theta_y+\gamma w(r|\boldsymbol{\phi},d,y)))\\
    &+P_y(cos(\theta_y+\gamma w(r|\boldsymbol{\phi},d,y)
     -cos(\theta_y+\gamma w(r|\boldsymbol{\phi},d,y)))]+\sum_{c=1}^C P_c(-log P_c).
\end{aligned}
\end{equation}
At $\textbf{P}=\textbf{P}^*$, the maximum of the right hand side (r.h.s.) is
\begin{equation}\label{eq:AIDGNloss}
\begin{aligned}
    &\ell_{AIDGN}(\textbf{z},y)\\
    &=-log \frac{e^{\kappa cos(\theta_y+\gamma w(r|\boldsymbol{\phi},d,y))}}{e^{\kappa cos(\theta_y+\gamma w(r|\boldsymbol{\phi},d,y))}+\sum_{c\neq y}e^{\kappa cos(\theta_y)}}.
\end{aligned}
\end{equation}
\end{theorem}

\begin{proof}
For ease of notation, we prove this theorem by introducing a lemma.
\begin{lemma}
For $a_c,b\in \mathbb{R}, c=1,\ldots,C$, $\textbf{P}^*=(\frac{exp(\kappa(a_1-b))}{\sum_{c=1}^C exp(\kappa(a_c-b))},\ldots,\frac{exp(\kappa(a_C-b))}{\sum_{c=1}^C exp(\kappa(a_c-b))})$ is the solution to the following constrained concave optimization problem.
\begin{equation}\label{eq:concave}
    \max \kappa[\sum_{c=1}^C P_c(a_c-b)]+\sum_{c=1}^C P_c\log(\frac{1}{P_c}), s.t. \sum_{c=1}P_c=1.
\end{equation}
\end{lemma}
\noindent \textit{Proof of lemma.}
\ Observing that there is only one equality constraint, we may converge the original constrained concave optimization problem eq.~\eqref{eq:concave} to a unconstrained convex optimization problem, we have the Lagrange multiplier
\begin{equation}\label{eq:convex}
     L(\textbf{P},\omega)= -\kappa[\sum_{c=1}^C P_c(a_c-b)]-\sum_{c=1}^C P_c\log(\frac{1}{P_c})+ \omega (\sum_{c=1}P_c-1).
\end{equation}
Then by the Karush Kuhn Tucker condition, we have
\begin{equation}
    \begin{cases}
    \partial  L(\textbf{P},\omega)/ \partial \omega =\sum_{c=1}P_c-1=0,\\
    \partial  L(\textbf{P},\omega)/ \partial P_c = -\kappa P_c(a_c-b)-1+\omega=0, c=1,\ldots,C.
    \end{cases}
\end{equation}
With some manipulations, we have
\begin{equation}
    \begin{cases}
    \sum_{c=1}P_c=1\\
    P_c=\frac{exp(\kappa(a_c-b))}{exp(1+\omega)}, c=1,\ldots,C.
    \end{cases}
\end{equation}
Combining those two equalities, we have
\begin{equation}
    \begin{cases}
    exp(1+\omega)=\sum_{c=1}^C exp(\kappa(a_c-b),\\
    P_c=\frac{exp(\kappa(a_c-b))}{\sum_{c=1}^C exp(\kappa(a_c-b)}, c=1,\ldots,C.
    \end{cases}
\end{equation}
Evaluating eq.\eqref{eq:concave} at $\textbf{P}^*=(\frac{exp(\kappa(a_1-b_1))}{\sum_{c=1}^C exp(\kappa(a_c-b))},\ldots,\frac{exp(\kappa(a_C-b))}{\sum_{c=1}^C exp(\kappa(a_c-b))})$, we have
\begin{equation}
    \max_{\textbf{P}\in \Delta} \kappa[\sum_{c=1}^C P_c(a_c-b)]+\sum_{c=1}^C P_c\log(\frac{1}{P_c})=\log(\sum_{c=1}^C exp(\kappa(a_c-b)))=-\log\frac{exp(\kappa b)}{\sum_{c=1}^C exp(\kappa a_c)}.
\end{equation}
Whence taking $a_y=b= cos(\theta_y+\gamma w(r|\boldsymbol{\phi},d,y))$ and $a_c=cos(\theta_c)$ for $c\neq y$ complete the proof.
\end{proof}
\clearpage

\section{Experimental details}
\subsection{Baseline details}
This appendix provides a detailed description of the 14 baseline methods used for benchmark comparisons. 
\begin{itemize}
    \item Empirical Risk Minimization (\textbf{ERM}, ~\cite{vapnik1999erm}) aggregates the data from all source domains together, and minimizes the cross entropy loss for classification.
    \item Invariant Risk Minimization (\textbf{IRM}, ~\cite{arjovsky2019irm}) learns a feature mapping such that the optimal linear classifier on top of that representation matches across source domains.
    \item Group Distributionally Robust Optimization (\textbf{DRO}, ~\cite{sagawa2019dro}) performs ERM while increasing the importance of domains with larger error by re-weighting minibatches.
    \item Inter-domain Mixup (\textbf{Mixup}, ~\cite{wang2020mixup}) employs mixup~\cite{zhang2018mixup_origin} technique across multiple domains and performs ERM on the augmented heterogeneous mixup distribution.
    \item Meta-Learning for Domain Generalization (\textbf{MLDG}, ~\cite{li2018mldg}) divides the source domains into meta-train-domains and meta-test-domain to simulate domain shift, and regulate the model trained on meta-train-domains to perform well on meta-test-domain.
    \item Deep CORrelation ALignment (\textbf{CORAL}, ~\cite{sun2016coral}) matches the first-order (mean) and the second-order (covariance) statistics of feature distributions across source domains.
    \item Maximum Mean Discrepancy (\textbf{MMD}, ~\cite{li2018mmd}) achieves distribution alignment in the latent space of an autoencoder by using adversarial learning and the maximum mean discrepancy criteria.
    \item Domain Adversarial Neural Network (\textbf{DANN}, ~\cite{ganin2016dann}) employs a domain discriminator to align feature distributions across domains using adversarial learning.
    \item Class-conditional Domain Adversarial Neural Network (\textbf{CDANN}, ~\cite{li2018cdann}) matches conditional feature distributions across domains, enabling alignment of multimodal distributions for all class labels.
    \item Marginal Transfer Learning (\textbf{MTL}, ~\cite{blanchard2021MTL}) estimates a kernel mean embedding per domain, passed as a second argument to the classifier. Then, these embeddings are estimated using single test examples at test time.
    \item Style Agnostic Networks (\textbf{SagNet}, ~\cite{nam2021sagnet}) disentangle style encodings from class categories to prevent style biased predictions and focus more on the contents.
    \item Adaptive Risk Minimization (\textbf{ARM}, ~\cite{zhang2020arm} is an extension of MLDG and introduces an additional module to compute domain embeddings, which are used by the prediction module to infer information about the input distribution.
    \item Variance Risk Extrapolation (\textbf{VREx}, ~\cite{krueger2021vrex}) is a form of robust optimization over a perturbation set of extrapolated domains and minimizes the variance of training risks across domains.
    \item Representation Self-Challenging (\textbf{RSC}, ~\cite{huang2020rsc}) iteratively discards the dominant features activated on the training data, and forces the CNN to activate remaining features that correlates with labels.
\end{itemize}

\subsection{Implementation details}
This appendix provides more implementation details about data preprocessing techniques, model architectures, objective function, hyperparameters and experimental environments. We follow similar settings as~\cite{gulrajani2020domainbed} for a fair comparison. 

\paragraph{Data preprocessing.} We use the same data preprocessing techniques for all the 4 datasets used in the experiments. Specifically, for training data, we use the following procedure same as~\cite{gulrajani2020domainbed}: crops of random size and aspect ratio, resizing to 224 $\times$ 224 $\times$ 3 pixels, random horizontal flips, random color jitter, grayscaling with 10\% probability, and normalization using the ImageNet channel statistics. For testing data, we only resize the image to 224 $\times$ 224 $\times$ 3 pixels and use the normalization by the ImageNet channel statistics.

\paragraph{Model architectures.} For a fair comparison, we use the Resnet-50~\cite{he2016resnet} model pre-trained on ImageNet~\cite{deng2009imagenet} as the feature mapping backbone $\pi$. We customize the final fully connected layer of ResNet-50 according to the number of categories of the datasets, and use it as the classifier $f$. We freeze all the batch normalization (BN) layers from pre-trained ResNet-50, since different domains in a mini-batch follow different distributions and BN degrades domain generalization performance. 

\paragraph{Objective function and Hyperparameters.} As stated in Section 4, to estimate $\delta_{y,\boldsymbol{\phi}}$ and $\mu_{d,y,\boldsymbol{\phi}}$ in AIDGN loss~\eqref{eq:finalOO}, we treat all $\delta_{y,\boldsymbol{\phi}}$ as one single hyperparameter $\delta$, and perform in-batch estimation for the $d$-th domain norm scale parameter, i.e., $\mu_d$, while ignoring the index $y$ and $\boldsymbol{\phi}$. Specifically, we estimate $\mu_d$ by computing the average norm of samples of the $d$-th domain in a minibatch. To compensate for the imprecise estimation of $\mu_{d,y,\boldsymbol{\phi}}$ by using $\mu_d$, we additionally introduce a hyperparameter $\beta$ to reweight the domain norm scale parameter. Moreover, for the log term in the KL regularizer, we approximate it with a first order Maclaurin's expansion to make the overall objective function convex. Thus, the optimization objective function used in our experimental implementation can be written as
\begin{equation}
    L_{AIDGN}=\sum_{d=1}^{N} \sum_{i=1}^{n_d} -log \frac{e^{\kappa cos(\theta_y +\gamma_{\delta}(r +\beta \mu_d))}}{e^{\kappa cos(\theta_y +\gamma_{\delta}(r +\beta \mu_d))}+\sum_{c\neq y}e^{\kappa cos(\theta_c)}} + \eta (\frac{\mu_d}{\mu^*}+\frac{\mu^*}{\mu_d})
\end{equation}

Thus, the AIDGN specific hyperparameters include $\kappa, \beta, \gamma_{\delta}(=\gamma/\delta)$, $\eta$ and $\mu^*$. Since the number of AIDGN specific hyperparamets is a somewhat more than other baseline methods, fixed number (e.g., 20) of trials for random search may not be sufficient to explore the joint hyperparameter distribution and is unfair to our method. On the other hand, a large number of random trials is computationally too expensive. Hence, following~\cite{cha2021swad}, we reduce the search space of AIDGN for computational efficiency. Specifically, batch size for each domain is fixed as 32. The dropout rate and weight decay are set as 0 and 1e-6, respectively. $\kappa$, $\gamma_{\delta}$ and $\mu^*$ are searched in PACS and the searched values are used as default settings for all the other datasets. We use $\kappa=110$, $\gamma_{\delta}=0.001$ and $\mu^*=410$. We search for 3 hyperparameters following standard hyperparameter search protocol in DomainBed~\cite{gulrajani2020domainbed}: we search learning rate in  $\{1e$-$5, 2.5e$-$5, 5e$-$5\}$, $\beta$ in $[0.05, 0.5]$ and $\eta$ in $[0.02, 0.06]$. The network is trained for 5000 iterations, which is enough to be converged. For PACS and OfficeHome, we divide the learning rate by 2 at 2000 and 4000 iterations, which can further improve validation accuracy for our model. We slightly modify the evaluation frequency since it should be small enough to exactly detect when the model is converged and overfitted. In consideration of exactness and efficiency, we set the evaluation frequency as 50 for PACS and VLCS, 100 for OfficeHome and TerraIncognita (All are 300 in DomainBed default settings).

\paragraph{Experimental environments.} For hardware environments, we perform our experiments on three machines: two with 8 Nvidia RTX3090s and Xeon E5-2680, and one with 4 Nvidia V100 and Xeon Platinum 8163. For software environments, our experiments are conducted with Python 3.7.9, and the following packages are used: PyTorch 1.7.1, torchvision 0.8.2 and NumPy 1.19.4.

\subsection{Full results of benchmark comparisons}
This appendix provides full results when compared with baseline methods in each benchmark dataset.

\subsubsection{PACS:}
\begin{table}[h]
\centering
\begin{tabular}{lccccc}
\toprule
 Method & A & C & P & S & Avg \\
\midrule
ERM~\cite{vapnik1999erm} & 84.7 $\pm$ 0.4 & 80.8 $\pm$ 0.6 & 97.2 $\pm$ 0.3 & 79.3 $\pm$ 1.0  & 85.5   \\
ERM$^\dagger$~\cite{vapnik1999erm} & 86.5 $\pm$ 0.8 & 79.9 $\pm$ 0.7 & 97.5 $\pm$ 0.1 & 79.3 $\pm$ 1.0  & 85.8 \\
IRM~\cite{arjovsky2019irm} & 84.8 $\pm$ 1.3       & 76.4 $\pm$ 1.1       & 96.7 $\pm$ 0.6       & 76.1 $\pm$ 1.0       & 83.5                \\
DRO~\cite{sagawa2019dro} & 83.5 $\pm$ 0.9       & 79.1 $\pm$ 0.6       & 96.7 $\pm$ 0.3       & 78.3 $\pm$ 2.0       & 84.4                 \\
Mixup~\cite{wang2020mixup} & 86.1 $\pm$ 0.5       & 78.9 $\pm$ 0.8       & {\bf97.6 $\pm$ 0.1}       & 75.8 $\pm$ 1.8       & 84.6               \\
MLDG~\cite{li2018mldg} & 85.5 $\pm$ 1.4       & 80.1 $\pm$ 1.7       & 97.4 $\pm$ 0.3       & 76.6 $\pm$ 1.1       & 84.9                \\
CORAL~\cite{sun2016coral} & {\bf88.3 $\pm$ 0.2}       & 80.0 $\pm$ 0.5       & 97.5 $\pm$ 0.3       & 78.8 $\pm$ 1.3       & 86.2               \\
MMD~\cite{li2018mmd} & 86.1 $\pm$ 1.4       & 79.4 $\pm$ 0.9       & 96.6 $\pm$ 0.2       & 76.5 $\pm$ 0.5       & 84.6                 \\
DANN~\cite{ganin2016dann} & 86.4 $\pm$ 0.8       & 77.4 $\pm$ 0.8       & 97.3 $\pm$ 0.4       & 73.5 $\pm$ 2.3       & 83.6            \\
CDANN~\cite{li2018cdann} & 84.6 $\pm$ 1.8       & 75.5 $\pm$ 0.9       & 96.8 $\pm$ 0.3       & 73.5 $\pm$ 0.6       & 82.6                \\
MTL~\cite{blanchard2021MTL} & 87.5 $\pm$ 0.8       & 77.1 $\pm$ 0.5       & 96.4 $\pm$ 0.8       & 77.3 $\pm$ 1.8       & 84.6                \\
SagNet~\cite{nam2021sagnet} & 87.4 $\pm$ 1.0       & 80.7 $\pm$ 0.6       & 97.1 $\pm$ 0.1       & {\bf80.0 $\pm$ 0.4}       & 86.3                 \\
ARM~\cite{zhang2020arm} & 86.8 $\pm$ 0.6       & 76.8 $\pm$ 0.5       & 97.4 $\pm$ 0.3       & 79.3 $\pm$ 1.2       & 85.1                \\
VREx~\cite{krueger2021vrex} & 86.0 $\pm$ 1.6       & 79.1 $\pm$ 0.6       & 96.9 $\pm$ 0.5       & 77.7 $\pm$ 1.7       & 84.9                \\
RSC~\cite{huang2020rsc} & 85.4 $\pm$ 0.8       & 79.7 $\pm$ 1.8       & {\bf97.6 $\pm$ 0.3}       & 78.2 $\pm$ 1.2       & 85.2                 \\
\hline
AIDGN (ours) &  87.9 $\pm$ 0.4 & {\bf82.1 $\pm$ 0.2} & {\bf97.6 $\pm$ 0.1} &  78.8 $\pm$ 0.9 & {\bf 86.6}   \\
\bottomrule
\end{tabular}
\caption{Out-of-domain accuracies(\%) on PACS.}
\label{full results: PACS}
\end{table}

\newpage

\subsubsection{VLCS:}
\begin{table}[h]
\centering
\begin{tabular}{lccccc}
\toprule
 Method & C & L & S & V & Avg \\
\midrule
ERM~\cite{vapnik1999erm} &  97.7 $\pm$ 0.4 & 64.3 $\pm$ 0.9 & 73.4 $\pm$ 0.5  & 74.6 $\pm$ 1.3  & 77.5  \\
ERM$^\dagger$~\cite{vapnik1999erm} & 98.6 $\pm$ 0.1 & 65.3 $\pm$ 1.4 & 71.2 $\pm$ 1.4  & 76.0 $\pm$ 0.1    & 77.7 \\
IRM~\cite{arjovsky2019irm} &  98.6 $\pm$ 0.1       & 64.9 $\pm$ 0.9       & 73.4 $\pm$ 0.6       & 77.3 $\pm$ 0.9       & 78.5                 \\
DRO~\cite{sagawa2019dro} & 97.3 $\pm$ 0.3       & 63.4 $\pm$ 0.9       & 69.5 $\pm$ 0.8       & 76.7 $\pm$ 0.7       & 76.7                 \\
Mixup~\cite{wang2020mixup} & 98.3 $\pm$ 0.6       & 64.8 $\pm$ 1.0       & 72.1 $\pm$ 0.5       & 74.3 $\pm$ 0.8       & 77.4                 \\
MLDG~\cite{li2018mldg} & 97.4 $\pm$ 0.2       & 65.2 $\pm$ 0.7       & 71.0 $\pm$ 1.4       & 75.3 $\pm$ 1.0       & 77.2                 \\
CORAL~\cite{sun2016coral} & 98.3 $\pm$ 0.1       & 66.1 $\pm$ 1.2       & 73.4 $\pm$ 0.3       & 77.5 $\pm$ 1.2       & 78.8                \\
MMD~\cite{li2018mmd} & 97.7 $\pm$ 0.1       & 64.0 $\pm$ 1.1       & 72.8 $\pm$ 0.2       & 75.3 $\pm$ 3.3       & 77.5                 \\
DANN~\cite{ganin2016dann} &  {\bf99.0 $\pm$ 0.3}       & 65.1 $\pm$ 1.4       & 73.1 $\pm$ 0.3       & 77.2 $\pm$ 0.6       & 78.6             \\
CDANN~\cite{li2018cdann} & 97.1 $\pm$ 0.3       & 65.1 $\pm$ 1.2       & 70.7 $\pm$ 0.8       & 77.1 $\pm$ 1.5       & 77.5                 \\
MTL~\cite{blanchard2021MTL} & 97.8 $\pm$ 0.4       & 64.3 $\pm$ 0.3       & 71.5 $\pm$ 0.7       & 75.3 $\pm$ 1.7       & 77.2                 \\
SagNet~\cite{nam2021sagnet} & 97.9 $\pm$ 0.4       & 64.5 $\pm$ 0.5       & 71.4 $\pm$ 1.3       & 77.5 $\pm$ 0.5       & 77.8                 \\
ARM~\cite{zhang2020arm} &  98.7 $\pm$ 0.2       & 63.6 $\pm$ 0.7       & 71.3 $\pm$ 1.2       & 76.7 $\pm$ 0.6       & 77.6                 \\
VREx~\cite{krueger2021vrex} & 98.4 $\pm$ 0.3       & 64.4 $\pm$ 1.4       & {\bf74.1 $\pm$ 0.4}       & 76.2 $\pm$ 1.3       & 78.3                 \\
RSC~\cite{huang2020rsc} &  97.9 $\pm$ 0.1       & 62.5 $\pm$ 0.7       & 72.3 $\pm$ 1.2       & 75.6 $\pm$ 0.8       & 77.1                 \\
\hline
AIDGN (ours) &  98.3 $\pm$ 0.2 & {\bf65.7 $\pm$ 0.4} & 73.1 $\pm$ 0.4 &  {\bf78.7 $\pm$ 0.7} & {\bf 78.9}  \\
\bottomrule
\end{tabular}
\caption{Out-of-domain accuracies(\%) on VLCS.}
\label{full results: VLCS}
\end{table}

\subsubsection{OfficeHome:}
\begin{table}[h]
\centering
\begin{tabular}{lccccc}
\toprule
 Method & A & C & P & R & Avg \\
\midrule
ERM~\cite{vapnik1999erm} & 61.3 $\pm$ 0.7       & 52.4 $\pm$ 0.3       & 75.8 $\pm$ 0.1       & 76.6 $\pm$ 0.3       & 66.5                 \\
ERM$^\dagger$~\cite{vapnik1999erm} & 61.2 $\pm$ 0.5       & 52.9 $\pm$ 0.2       & 76.1 $\pm$ 0.4       & 77.5 $\pm$ 0.1       & 66.9  \\
IRM~\cite{arjovsky2019irm} & 58.9 $\pm$ 2.3       & 52.2 $\pm$ 1.6       & 72.1 $\pm$ 2.9       & 74.0 $\pm$ 2.5       & 64.3                \\
DRO~\cite{sagawa2019dro} & 60.4 $\pm$ 0.7       & 52.7 $\pm$ 1.0       & 75.0 $\pm$ 0.7       & 76.0 $\pm$ 0.7       & 66.0                 \\
Mixup~\cite{wang2020mixup} & 62.4 $\pm$ 0.8       & {\bf54.8 $\pm$ 0.6}       & {\bf76.9 $\pm$ 0.3}       & 78.3 $\pm$ 0.2       & 68.1                  \\
MLDG~\cite{li2018mldg} & 61.5 $\pm$ 0.9       & 53.2 $\pm$ 0.6       & 75.0 $\pm$ 1.2       & 77.5 $\pm$ 0.4       & 66.8                \\
CORAL~\cite{sun2016coral} & {\bf65.3 $\pm$ 0.4}       & 54.4 $\pm$ 0.5       & 76.5 $\pm$ 0.1       & 78.4 $\pm$ 0.5       & 68.7               \\
MMD~\cite{li2018mmd} & 60.4 $\pm$ 0.2       & 53.3 $\pm$ 0.3       & 74.3 $\pm$ 0.1       & 77.4 $\pm$ 0.6       & 66.3               \\
DANN~\cite{ganin2016dann} & 59.9 $\pm$ 1.3       & 53.0 $\pm$ 0.3       & 73.6 $\pm$ 0.7       & 76.9 $\pm$ 0.5       & 65.9                 \\
CDANN~\cite{li2018cdann} & 61.5 $\pm$ 1.4       & 50.4 $\pm$ 2.4       & 74.4 $\pm$ 0.9       & 76.6 $\pm$ 0.8       & 65.8                 \\
MTL~\cite{blanchard2021MTL} & 61.5 $\pm$ 0.7       & 52.4 $\pm$ 0.6       & 74.9 $\pm$ 0.4       & 76.8 $\pm$ 0.4       & 66.4              \\
SagNet~\cite{nam2021sagnet} & 63.4 $\pm$ 0.2       & {\bf54.8 $\pm$ 0.4}       & 75.8 $\pm$ 0.4       & 78.3 $\pm$ 0.3       & 68.1                \\
ARM~\cite{zhang2020arm} & 58.9 $\pm$ 0.8       & 51.0 $\pm$ 0.5       & 74.1 $\pm$ 0.1       & 75.2 $\pm$ 0.3       & 64.8               \\
VREx~\cite{krueger2021vrex} & 60.7 $\pm$ 0.9       & 53.0 $\pm$ 0.9       & 75.3 $\pm$ 0.1       & 76.6 $\pm$ 0.5       & 66.4               \\
RSC~\cite{huang2020rsc} & 60.7 $\pm$ 1.4       & 51.4 $\pm$ 0.3       & 74.8 $\pm$ 1.1       & 75.1 $\pm$ 1.3       & 65.5                \\
\hline
AIDGN (ours) & 64.9 $\pm$ 0.3 & 54.7 $\pm$ 0.3 & 76.5 $\pm$ 0.2 & {\bf79.1 $\pm$ 0.2} & {\bf68.8} \\
\bottomrule
\end{tabular}
\caption{Out-of-domain accuracies(\%) on OfficeHome.}
\label{full results: OfficeHome}
\end{table}

\newpage

\subsubsection{TerraIncognita:}
\begin{table}[h]
\centering
\begin{tabular}{lccccc}
\toprule
 Method & L100 & L38 & L43 & L46 & Avg \\
\midrule
ERM~\cite{vapnik1999erm} & 49.8 $\pm$ 4.4       & 42.1 $\pm$ 1.4       & 56.9 $\pm$ 1.8       & 35.7 $\pm$ 3.9       & 46.1                 \\
ERM$^\dagger$~\cite{vapnik1999erm} & 48.6 $\pm$ 0.3       & 43.5 $\pm$ 1.4       & 54.2 $\pm$ 1.0       & 36.9 $\pm$ 0.3       & 45.8\\
IRM~\cite{arjovsky2019irm} & 54.6 $\pm$ 1.3       & 39.8 $\pm$ 1.9       & 56.2 $\pm$ 1.8       & 39.6 $\pm$ 0.8       & 47.6                 \\
DRO~\cite{sagawa2019dro} &  41.2 $\pm$ 0.7       & 38.6 $\pm$ 2.1       & 56.7 $\pm$ 0.9       & 36.4 $\pm$ 2.1       & 43.2                 \\
Mixup~\cite{wang2020mixup} & {\bf59.6 $\pm$ 2.0}       & 42.2 $\pm$ 1.4       & 55.9 $\pm$ 0.8       & 33.9 $\pm$ 1.4       & 47.9                 \\
MLDG~\cite{li2018mldg} & 54.2 $\pm$ 3.0       & 44.3 $\pm$ 1.1       & 55.6 $\pm$ 0.3       & 36.9 $\pm$ 2.2       & 47.7                 \\
CORAL~\cite{sun2016coral} & 51.6 $\pm$ 2.4       & 42.2 $\pm$ 1.0       & 57.0 $\pm$ 1.0       & 39.8 $\pm$ 2.9       & 47.6                 \\
MMD~\cite{li2018mmd} & 41.9 $\pm$ 3.0       & 34.8 $\pm$ 1.0       & 57.0 $\pm$ 1.9       & 35.2 $\pm$ 1.8       & 42.2                 \\
DANN~\cite{ganin2016dann} &  51.1 $\pm$ 3.5       & 40.6 $\pm$ 0.6       & 57.4 $\pm$ 0.5       & 37.7 $\pm$ 1.8       & 46.7                 \\
CDANN~\cite{li2018cdann} & 47.0 $\pm$ 1.9       & 41.3 $\pm$ 4.8       & 54.9 $\pm$ 1.7       & 39.8 $\pm$ 2.3       & 45.8                 \\
MTL~\cite{blanchard2021MTL} & 49.3 $\pm$ 1.2       & 39.6 $\pm$ 6.3       & 55.6 $\pm$ 1.1       & 37.8 $\pm$ 0.8       & 45.6                 \\
SagNet~\cite{nam2021sagnet} & 53.0 $\pm$ 2.9       & 43.0 $\pm$ 2.5       & {\bf57.9 $\pm$ 0.6}       & 40.4 $\pm$ 1.3       & 48.6                 \\
ARM~\cite{zhang2020arm} & 49.3 $\pm$ 0.7       & 38.3 $\pm$ 2.4       & 55.8 $\pm$ 0.8       & 38.7 $\pm$ 1.3       & 45.5                 \\
VREx~\cite{krueger2021vrex} & 48.2 $\pm$ 4.3       & 41.7 $\pm$ 1.3       & 56.8 $\pm$ 0.8       & 38.7 $\pm$ 3.1       & 46.4                 \\
RSC~\cite{huang2020rsc} & 50.2 $\pm$ 2.2       & 39.2 $\pm$ 1.4       & 56.3 $\pm$ 1.4       & {\bf40.8 $\pm$ 0.6}       & 46.6                 \\
\hline
AIDGN (ours) & 54.7 $\pm$ 1.8 & {\bf46.8 $\pm$ 1.1} & 57.6 $\pm$ 0.8 & 38.5 $\pm$ 0.8 & {\bf49.4}  \\
\bottomrule
\end{tabular}
\caption{Out-of-domain accuracies(\%) on TerraIncognita.}
\label{full results: TerraIncognita}
\end{table}

\subsection{Full results of test-time adaptation}
This appendix shows full results of test-time adaptation on the four datasets, respectively.

\subsubsection{PACS:}
\begin{table}[h]
\centering
\begin{tabular}{lccccc}
\toprule
  & A & C & P & S & Avg \\
\midrule
AIDGN & 87.9 $\pm$ 0.4 & 82.1 $\pm$ 0.2 & 97.6 $\pm$ 0.1 & 78.8 $\pm$ 0.9 & 86.6 \\
+ T3A & 87.9 $\pm$ 0.5 & 82.1 $\pm$ 0.2 & 97.6 $\pm$ 0.1 & 78.7 $\pm$ 0.9 & 86.6  \\
+ Tent-BN & 86.5 $\pm$ 0.6 & 81.0 $\pm$ 0.3 & 97.6 $\pm$ 0.1 & 78.0 $\pm$ 0.9 & 85.8  \\
+ Tent-C   & 87.9 $\pm$ 0.4 & 76.7 $\pm$ 2.5 & 97.4 $\pm$ 0.2 & 74.1 $\pm$ 2.6 & 84.0  \\

\bottomrule
\end{tabular}
\caption{Out-of-domain accuracies(\%) on PACS with test-time adaptation.}
\label{tta:PACS}
\end{table}

\subsubsection{VLCS:}
\begin{table}[h]
\centering
\begin{tabular}{lccccc}
\toprule
  & C & L & S & V & Avg \\
\midrule
AIDGN & 98.3 $\pm$ 0.2 & 65.7 $\pm$ 0.4 & 73.1 $\pm$ 0.4 & 78.7 $\pm$ 0.7 & 78.9 \\
+ T3A & 98.3 $\pm$ 0.3 & 65.8 $\pm$ 0.4 & 72.9 $\pm$ 0.6 & 78.7 $\pm$ 0.6 & 78.9  \\
+ Tent-BN & 95.0 $\pm$ 1.6 & 57.2 $\pm$ 0.1 & 62.6 $\pm$ 1.5 & 71.8 $\pm$ 0.4 & 71.7  \\
+ Tent-C   & 98.3 $\pm$ 0.1 & 62.1 $\pm$ 0.5 & 71.4 $\pm$ 1.6 & 66.1 $\pm$ 4.7 & 74.5  \\

\bottomrule
\end{tabular}
\caption{Out-of-domain accuracies(\%) on VLCS with test-time adaptation.}
\label{tta:VLCS}
\end{table}

\newpage

\subsubsection{OfficeHome:}
\begin{table}[h]
\centering
\begin{tabular}{lccccc}
\toprule
  & A & C & P & R & Avg \\
\midrule
AIDGN & 64.9 $\pm$ 0.3 & 54.7 $\pm$ 0.3 & 76.5 $\pm$ 0.2 & 79.1 $\pm$ 0.2 & 68.8 \\
+ T3A & 64.9 $\pm$ 0.2 & 55.1 $\pm$ 0.2 & 76.6 $\pm$ 0.2 & 79.1 $\pm$ 0.1 & 68.9  \\
+ Tent-BN & 64.5 $\pm$ 0.2 & 55.0 $\pm$ 0.3 & 76.1 $\pm$ 0.2 & 78.8 $\pm$ 0.1 & 68.6  \\
+ Tent-C   & 64.8 $\pm$ 0.2 & 53.5 $\pm$ 0.7 & 76.3 $\pm$ 0.2 & 78.3 $\pm$ 0.2 & 68.2  \\

\bottomrule
\end{tabular}
\caption{Out-of-domain accuracies(\%) on OfficeHome with test-time adaptation.}
\label{tta:OfficeHome}
\end{table}

\subsubsection{TerraIncognita:}
\begin{table}[h]
\centering
\begin{tabular}{lccccc}
\toprule
  & L100 & L38 & L43 & L46 & Avg \\
\midrule
AIDGN & 54.7 $\pm$ 1.8 & 46.8 $\pm$ 1.1 & 57.6 $\pm$ 0.8 & 38.5 $\pm$ 0.8 & 49.4 \\
+ T3A & 49.4 $\pm$ 1.3 & 48.8 $\pm$ 1.6 & 53.7 $\pm$ 1.7 & 35.9 $\pm$ 2.0 & 46.9  \\
+ Tent-BN & 46.6 $\pm$ 0.8 & 43.7 $\pm$ 0.4 & 48.4 $\pm$ 0.6 & 35.0 $\pm$ 0.4 & 43.4  \\
+ Tent-C   & 55.8 $\pm$ 2.6 & 46.3 $\pm$ 0.9 & 57.6 $\pm$ 0.8 & 38.3 $\pm$ 0.7 & 49.5  \\

\bottomrule
\end{tabular}
\caption{Out-of-domain accuracies(\%) on TerraIncognita with test-time adaptation.}
\label{tta:TerraIncognita}
\end{table}

\end{document}